\documentclass{article} 
\usepackage[preprint]{neurips_2019}
\usepackage{times}

\usepackage{amsmath,amsfonts,bm}

\def\eqref#1{equation~\ref{#1}}

\def\1{\bm{1}}

\DeclareMathAlphabet{\mathsfit}{\encodingdefault}{\sfdefault}{m}{sl}
\SetMathAlphabet{\mathsfit}{bold}{\encodingdefault}{\sfdefault}{bx}{n}

\DeclareMathOperator{\sign}{sign}

\usepackage[bookmarks=false]{hyperref}
\usepackage{url}

\usepackage{amsmath}
\usepackage{amssymb}
\usepackage{amsfonts}
\usepackage{hyperref}
\usepackage{algpseudocode}
\usepackage{algorithm}

\usepackage{subfigure}
\usepackage{graphicx}
\usepackage{graphicx}
\usepackage{multicol}
\usepackage{enumitem}
\usepackage{booktabs} 
\usepackage{amsthm}
\usepackage{todonotes}
\usepackage{xcolor}
\newcommand{\rpm}{\sbox0{$1$}\sbox2{$\scriptstyle\pm$}
  \raise\dimexpr(\ht0-\ht2)/2\relax\box2 }
\usepackage{multirow}
\usepackage{enumitem}
\usepackage{microtype}
\usepackage{booktabs} 

\newcommand{\eat}[1]{}
\newtheorem{theorem}{Theorem}
\newtheorem{proposition}{Proposition}

\newtheorem{lemma}[theorem]{Lemma}

\title{Differential Equation Units: Learning Functional Forms of Activation Functions from Data}

\author{MohamadAli Torkamani \\
Amazon Web Services \\
\texttt{alitor@amazon.com} \\
\And
Shiv Shankar \\
University of Massachusetts Amherst \\
\texttt{sshankar@cs.umass.edu}\\
\And 
Amirmohammad Rooshenas \\
University of Massachusetts Amherst \\
\texttt{pedram@cs.umass.edu} \\
\And
Phillip Wallis \\
Microsoft Inc. \\
\textit{phwallis@microsoft.com}
}

\begin{document}

\maketitle

\begin{abstract}
Most deep neural networks use simple, fixed activation functions, such
as sigmoids or rectified linear units, regardless of domain or
network structure. We introduce differential equation units (DEUs), an
improvement to modern neural networks, which enables each neuron to learn a particular nonlinear activation function from a family of solutions to an ordinary differential equation. Specifically, each neuron may change its functional form during training based on the behavior of the other parts of the network.
We show that using neurons with DEU activation functions results in a more compact network capable of achieving comparable, if not superior, performance when is compared to much larger
networks.

\end{abstract}

\section{Introduction}
\label{intro}
Driven in large part by advancements in storage, processing, and parallel computing, deep neural networks (DNNs) have become capable of outperforming other methods across a wide range of highly complex tasks. 

The advent of new activation functions such as rectified linear units (ReLU) \cite{nair2010rectified}, exponential linear units (ELU) \cite{clevert2015fast}, and scaled exponential linear units (SELU) \cite{klambauer2017self} address a network's ability to effectively learn complicated functions, thereby allowing them to perform better on complicated tasks. The choice of an activation function is typically determined empirically by tuning, or due to necessity. For example, in modern deep networks, ReLU activation functions are often favored over  sigmoid functions, which used to be a popular choice in the earlier days of neural networks. A reason for this preference is that the ReLU function is non-saturating and does not have the vanishing gradient problem when used in deep structures \cite{hochreiter1998vanishing}.

In all of the aforementioned activation functions, the functional form of the activation function is fixed. However, depending on data, different forms of activation functions may be more suitable to describe data. 
In this paper, we introduce differential equation units (DEUs), where the activation function of each neuron is the nonlinear, possibly periodic solution of a parameterized second order, linear, ordinary differential equation. 
By learning the parameters of the differential equations using gradient descent on the loss function, the solutions to the differential equations change their functional form, thus adapting their forms to extract more complex features from data.

From an applicability perspective, our approach is similar to Maxout networks \cite{goodfellow2013maxout} and adaptive piece-wise linear units (PLUs)~\cite{agostinelli2014learning,ramachandran2017searching}. 

Figure~\ref{fig:region} shows nested noisy circles for which we trained neural networks with Maxout, ReLU, and our proposed DEU activation functions. For using ReLU, we needed at least two layers of four hidden units to separate the circles, while one layer of either two Maxout or two DEUs is enough to learn an appropriate decision boundary.
While, the number of parameters learned by Maxout and PLU is proportional to the number of input weights to a neuron, and the number of linear units in that neuron, for each DEU, we learn {\it only five additional} parameters that give us highly flexible activation functions in return; altogether, resulting in more compact representations.
Moreover, DEUs can represent a broad range of functions including harmonic functions, while Maxout is limited to piecewise approximation of convex functions (see Figure~\ref{complex_fcn}).

DEUs can also transform their forms during training in response to the behavior of other neurons in the network in order to describe data with interaction to the other part of the network, thus the neurons in a neural network may adopt different forms as their activation functions. The variety of activation function forms throughout a network enables it to encode more information, thus requiring less neurons for achieving the same performance comparing to the networks with fixed activation functions.

The main contribution of this paper is to explore learning the form of activation functions through introducing differential equation units (DEUs). We also propose a learning algorithm to train the parameters of DEUs, and we empirically show that neural networks with DEUs can achieve high performance with more compact representations and are effective for solving real-world problems. 









\section{Differential Equation Units}

Inspired by functional analysis and calculus of variations, instead of using a fixed activation function for each layer in a network, we propose a novel solution for learning an activation function for each neuron. Experimentally, we show that by allowing each neuron to learn its own activation function, the network as a whole can perform on par with (or even outperform) much larger baseline networks. In the supplementary material, we have briefly outlined how in Calculus of Variations, the Euler-Lagrange formula can be used to derive a differential equation from an optimization problem with function variables~\cite{gel1963variatsionnoeischislenie,gelfand2000calculus}. 

\begin{figure}[t]
    \centering
    \includegraphics[width=0.8\columnwidth]{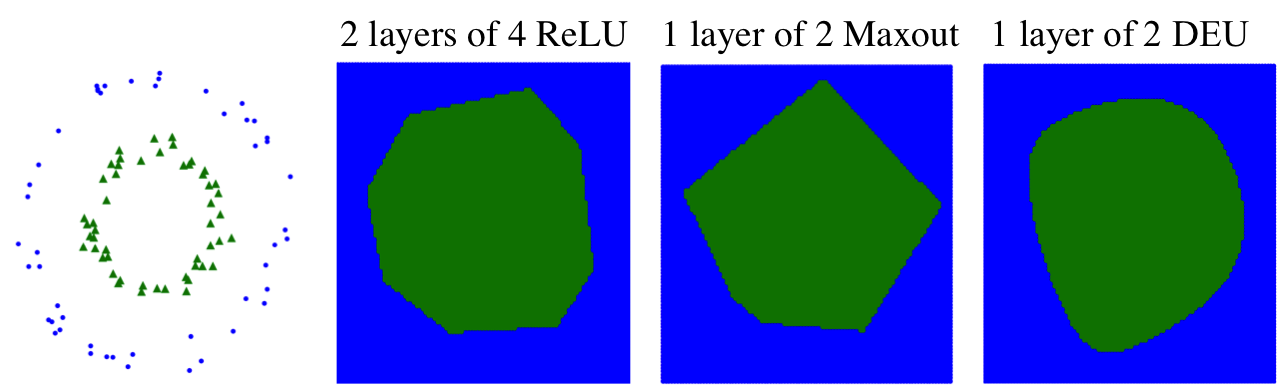}
    \caption{Learned decision boundaries by neural networks with ReLU, Maxout, and DEU activation functions for classifying noisy circles.}
    \label{fig:region}
\end{figure}

 \begin{figure*}[t]
\centering
\subfigure{\includegraphics[width=0.30\linewidth]{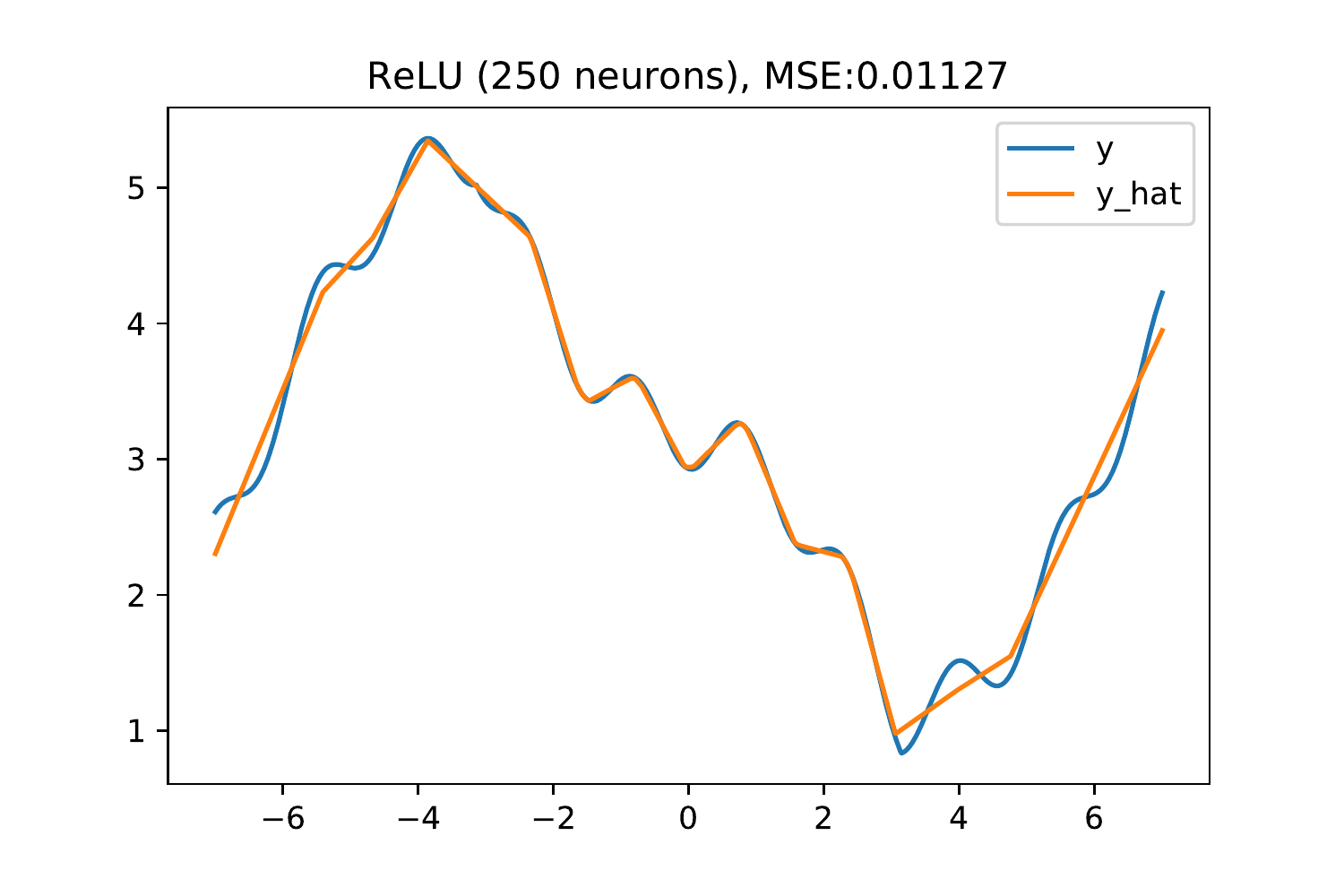}}
\subfigure{\includegraphics[width=0.30\linewidth]{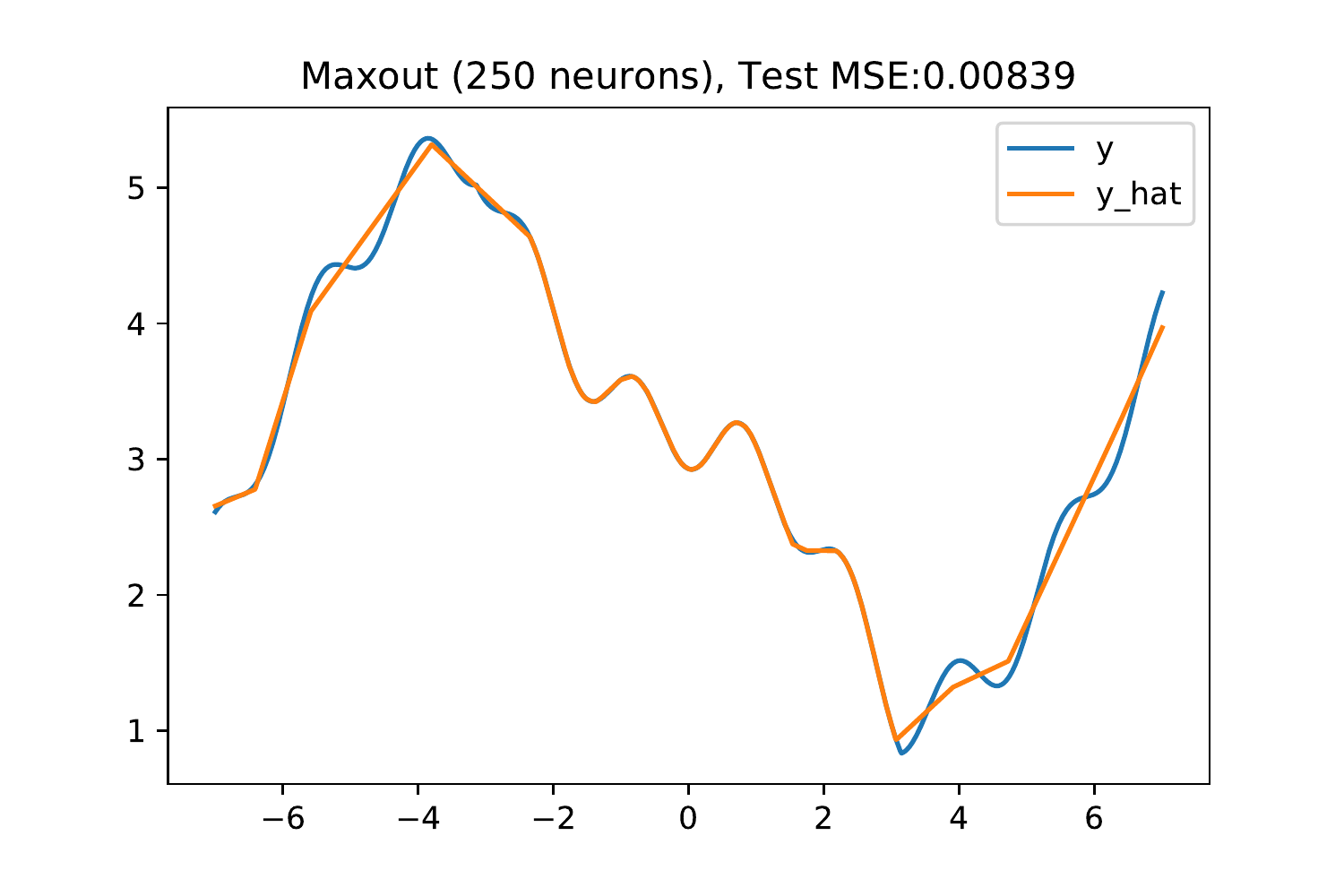}}
\subfigure{\includegraphics[width=0.30\linewidth]{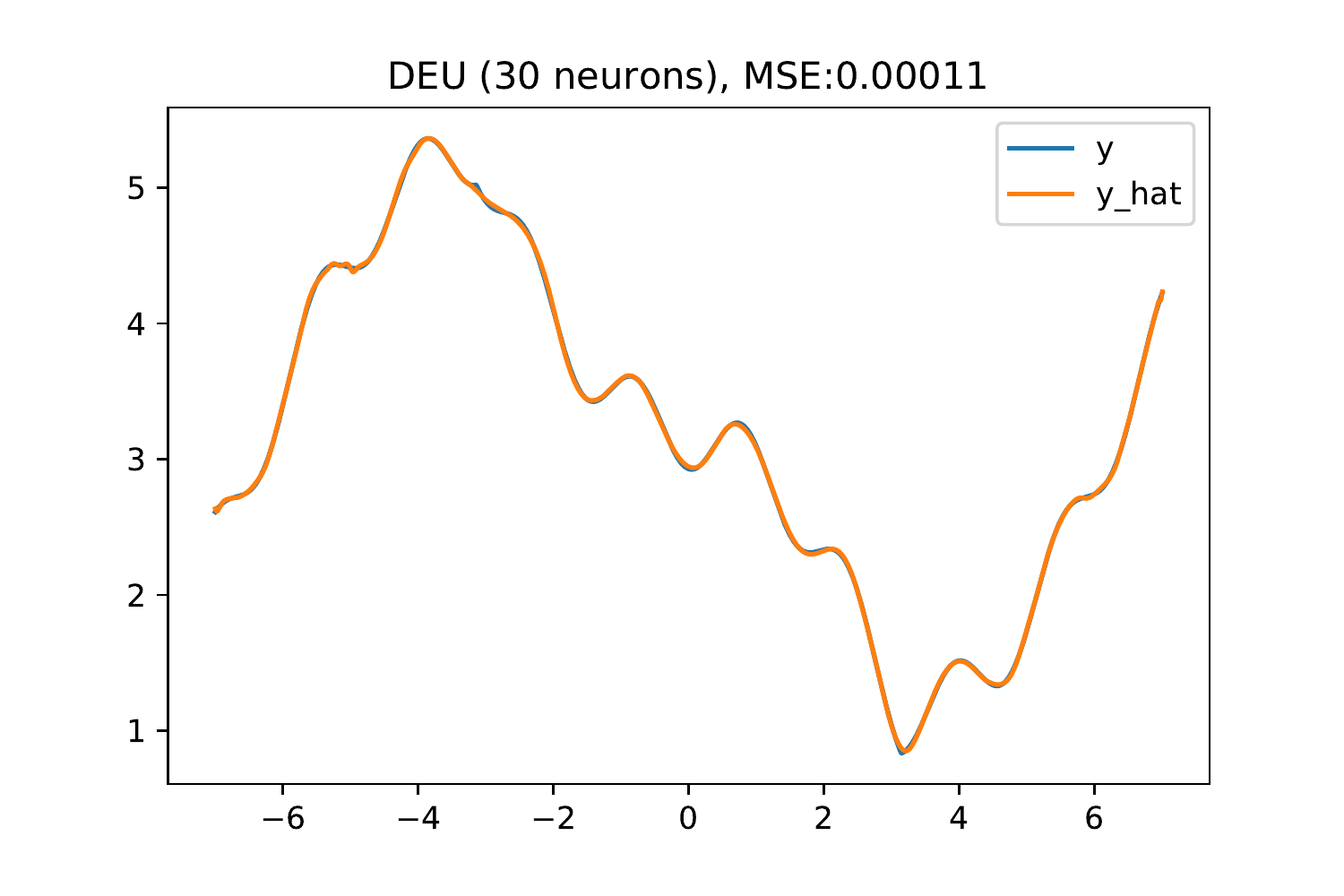}}
\caption{The problem of fitting a challenging function ($y = (\sin(t)-\cos(2t)^2)/2+ 4*(1+\arccos(\sin(t/2)))/3$). A neural network with 25 DEUs in one hidden layer fits the function much more accurately than neural networks of 250 neurons with ReLU and Maxout activation function.
\label{complex_fcn}
}
\end{figure*}


The main idea behind DEUs is to find the parameters of an ordinary differential equation (ODE) for each neuron in the network, whose solution would be used as the activation function of the corresponding neuron. As a result, each neuron learns a personalized activation function. We select (learn) the parameters of the differential equation from a low dimensional space (i.e. five dimensions). By minimizing the networks loss function, our learning algorithm smoothly updates the parameters of the ODE, which results in an uncountably\footnote{Up to computational precision limitations.} extensive range of possible activation functions.

We parameterize the activation function of each neuron using a linear, second order ordinary differential equation $a y''(t) + b y'(t) + c y(t) = g(t)$, parameterized by five coefficients ($a$, $b$, $c$, $c_1$, $c_2$), where $a$, $b$, and $c$ are scalars used to parameterize the ODE, $c_1$ and $c_2$ represent the initial conditions of the ODE's solution, and $g(t)$ is a regulatory function that we call the \textit{core activation function}. The coefficients are the only additional parameters that we learn for each neuron and are trained through the backpropagation algorithm.  To simplify the math and because it is a standard practice in control theory, we have set $g(t)=u(t)$, where $u(t)$ is the Heaviside (unit) step function. 
Therefore, the ODE that we have chosen has the following form: 
\begin{align}\label{diffeq_activation}
a y''(t) + b y'(t) + c y(t) = u(t), \nonumber\\
\text{~~ where } u(t)=
\begin{cases}
 0 & x\leq 0 \\
 1 & x>0\\
\end{cases}
\end{align}



This model is often used to denote the exchange of energy between mass and stiffness elements in a mechanical system or between capacitors and inductors in an electrical system~\cite{ogata2002modern}. Interestingly, by using the solutions of this formulation as activation functions, we can gain a few key properties: approximation or reduction to some of the standard activation functions such as sigmoid or ReLU; the ability to capture oscillatory forms; and, exponential decay or growth.

\subsection{The learning algorithm}
For fixed $a$, $b$ and $c$, the solution of the differential equation
is: 
\begin{align}
 y(t) = f(t;a,b,c) + c_1 f_1(t;a,b,c)+ c_2 f_2(t;a,b,c),   
\end{align}
for some functions $f$, $f_1$, $f_2$. Here, $y(t)$ lies on an affine space parameterized by scalars $c_1$ and $c_2$ that represent the initial conditions of the solution (for desired values of $y_0$ and $y'_0$ at some $t=t_0$ such that $y(t_0) = y_0$ and $\frac{\partial y(t)}{\partial t}\rvert_{t=t_0} = y'_0$).

\subsubsection{Closed-form solutions}
First, we solve the differential equations parametrically and take derivatives of the closed-form solutions: $\frac{\partial y}{\partial t}$ with respect to its input $t$, and $\frac{\partial y}{\partial a}$, $\frac{\partial y}{\partial b}$, and $\frac{\partial y}{\partial c}$ with respect to parameter $a$, $b$, and $c$. Moreover, the derivatives with respect to $c_1$ and $c_2$ are $f_1$ and $f_2$, respectively. This is done once. We have solved the equations and taken their derivatives using the software package \citeauthor{maplesoft}~(\citeyear{maplesoft}). Maple also generates optimized code for the solutions, by breaking down equations in order to reuse computations.\footnote{The code is available at \url{https://github.com/rooshenas/deu}} Although we used Maple here, this task could have been done by pen and paper.\\

\subsubsection{Training the parameters}
The parameters of the activation function (the ODE parameters and the appropriate initial conditions of its solution) are jointly trained with the  neural networks' parameters using back-propagation.\\
We adopt regular backpropagation to update the values of parameters $a$, $b$, $c$, $c_1$ and $c_2$ for each neuron, along with using $\frac{\partial y}{\partial t}$ for updating network parameters $w$ (input weights to the neuron), and for backpropagating the error to the lower layers.

We initialize network parameters using current best practices with respect to the layer type (e.g. linear layer, convolutional layer, etc.). We initialize parameters  $a$, $b$, $c$ for all neurons with a random positive number less than one, and strictly greater than zero. We initialize $c_1=c_2=0.0$. To learn the parameters $\mathbf{\theta} = [a, b, c, c_1, c_2]^T$ along with
the weights $\mathbf{w}$ on input values to each neuron, we deploy a
gradient descent algorithm. Both the weights $\mathbf{w}$, as well as $\mathbf{\theta}$ are learned using
the conventional backpropagation algorithm with Adam updates \cite{kingma2014adam}. 
During training, we treat $a, b, c, c_1$ and $c_2$ like biases to the neuron (i.e., with input weight of $1.0$)
and update their values based on the direction of the corresponding
gradients in each mini-batch.

\begin{table*}[!tbh]
\caption{Subspace and Singularities}
\label{singularities}
\centering
\begin{tabular}{ll}
 \toprule
Subspace & Solution\\
 \midrule
 $a=0,b=0,c\neq 0$ & $\sigma(t)/c \quad \text{(i.e.,~sigmoid when } c=1)$\\
 $a=0,b\neq 0,c=0$ & $x u(x)/b + c_1 \quad \text{(i.e.,~ReLU when } b=1 \text{~and~}c_1=0)$ \\
 $a=0,b\neq 0,c\neq 0$ & $c_1 e^{-(c x)/b} - u(x) e^{-(c x)/b}/c + u(x)/c$\\
 $a\neq 0,b=0,c= 0$ & $\frac{x^2 u(x)}{2 a} + c_2 x + c_1$ \\
 $a<0,b=0,c<0  \text{~or~} $ &\multirow{2}{*}{ $c_2\sin( \sqrt{\frac {c}{a}}x) +c_1 \cos(\sqrt{\frac {c}{a}}x) -\frac {u(x) }{c} \left( \cos(\sqrt{\frac {c}{a}}x) -1 \right)$} \\
 $a>0,b=0,c>0$&\\
 
 $a<0,b=0,c>0 \text{~or~}$& \multirow{2}{*}{$c_1 e^{\sqrt{-\frac{c}{a}x}} + c_2 e^{-\sqrt{-\frac{c}{a}x}} +\frac{u(x) e^{-\sqrt{-\frac{c}{a}x} }}{2 c} + \frac{u(x) e^{\sqrt{-\frac{c}{a}x} }}{2 c} - \frac{u(x)}{c}$}\\
  $a>0,b=0,c<0$& \\
  
 $a\neq 0,b\neq 0,c =0$ & ${\frac {u \left( x \right) a}{{b}^{2}}{{\rm e}^{-{\frac {bx}{a}}}}}-{
\frac {u \left( x \right) a}{{b}^{2}}}-c_1{\frac {a}{b}{
{\rm e}^{-{\frac {bx}{a}}}}}+{\frac {u \left( x \right) x}{b}}+{\it 
c_2}$\\
 $a\neq 0,b\neq 0,c\neq 0$ & \text{Four forms based on the sign of $\Delta = b^2 -4ac$,~$a$ and $c$}\\
 \bottomrule
\normalsize
\end{tabular}
\label{table_uci}
\end{table*}


\subsubsection{Singularity of solutions} If one or two of the coefficients $a$, $b$ or $c$ become zero, then the solution of the differential equation falls into a singularity subspace that is different than the affine function space of neighboring positive or negative values for those coefficients. For example, for $b=0$ and $a*c>0$ , the solution will be $y(t) = \sin( \sqrt{\frac {c}{a}}t) c_2+\cos(\sqrt{\frac {c}{a}}t) c_1-\frac {u(t) }{c} \left( \cos(\sqrt{\frac {c}{a}}t) -1 \right)$, but for $b=c=0$ and $a\neq 0$, the solution has the form of $y(t)=\frac{1}{2}{\frac {u(t) {t}^{2}}{a}}+c_1t+c_2$. In this example, we observe that moving from $c>0$ to $c=0$ changes the resulting activation function from a pure oscillatory form to a parametric (leaky) rectified quadratic activation function. More formally, for a parametric variable quantity $p\in\{a,b,c, b^2-4ac\}$, if $p=0$, then the solution of the differential equation may be different than the solution of the neighboring differential equations with $p\neq0$ (see Table~\ref{singularities} for the complete set of singular subspaces). Moreover, with $p \rightarrow 0$, then we may have $y(t;p)\rightarrow \infty$ for certain values of $t$. In particular, when exponential forms are involved in the solution, this phenomenon can cause an extreme increase in the magnitude of the output value of the DEU. Therefore, we introduce a hyperparameter $\epsilon$ which is used as a threshold to clamp small values to zero. This simple trick avoids numerical computation errors as well as exponentiating large positive numbers when the denominator of the exponentiated fraction is very small.


Our learning algorithm allows the activation functions to jump over singular subspaces. However, if the absolute value of $a$,$b$, or$c$ falls below $\epsilon$, then the derivative with respect to that parameter becomes zero. The value of the parameter remains zero during the training if we use the regular derivative. In order to allow an activation function to ``escape" a singular subspace, we introduce the concept of ``outward gravitation" in the next subsection.


We do not allow $a=b=c=0$, and in this rare case, we set $b$ to $\epsilon$. During the learning process at most two of $a$, $b$, and $c$ can be zero. The sign of $a$, $b$, $c$, and $b^2-4ac$ \eat{(with $a, b, c \in \mathbb{R}$) } might also change the solution of the ODE, which create ``subspaces" that are individually solved in closed-form. 

When $b^2-4ac$  is close to zero and $ac>0$, the generic solution may become exponentially large. Therefore if $-\epsilon<b^2-4ac<\epsilon$ and $\sign(a)==\sign(c)$, we explicitly assume $a = \frac{b}{2}\sign(a)$ $c = \frac{b}{2}\sign(c)$ in our implementation to stabilize the solution and to avoid large function values (i.e., we force $b^2-4ac=0$ in the solution of the ODE.).\\

\subsubsection{Approximation of Dirac's delta function} The derivatives of the activation function with respect to $t$ include Dirac's delta function $\delta(t)$, which is the derivative of the Heaviside function. In the parametric derivatives, we substitute the delta function with its approximation $s \frac{e^{-s*t}}{(1+e^{-s*t})^2}$, which is the derivative of $\sigma(s*t)=\frac{1}{(1+e^{-s*t})}$. This approximation is a commonly used in practice for the delta function~\cite{zahedi2010delta}. The larger values of $s$ result in more accurate approximation of the delta function. In all of our experiments, we set $\epsilon = .01$, and $s =100$ although further tuning might improve the results.



\begin{figure*}[tbh]
\centering
\subfigure{\includegraphics[width=0.3\linewidth]{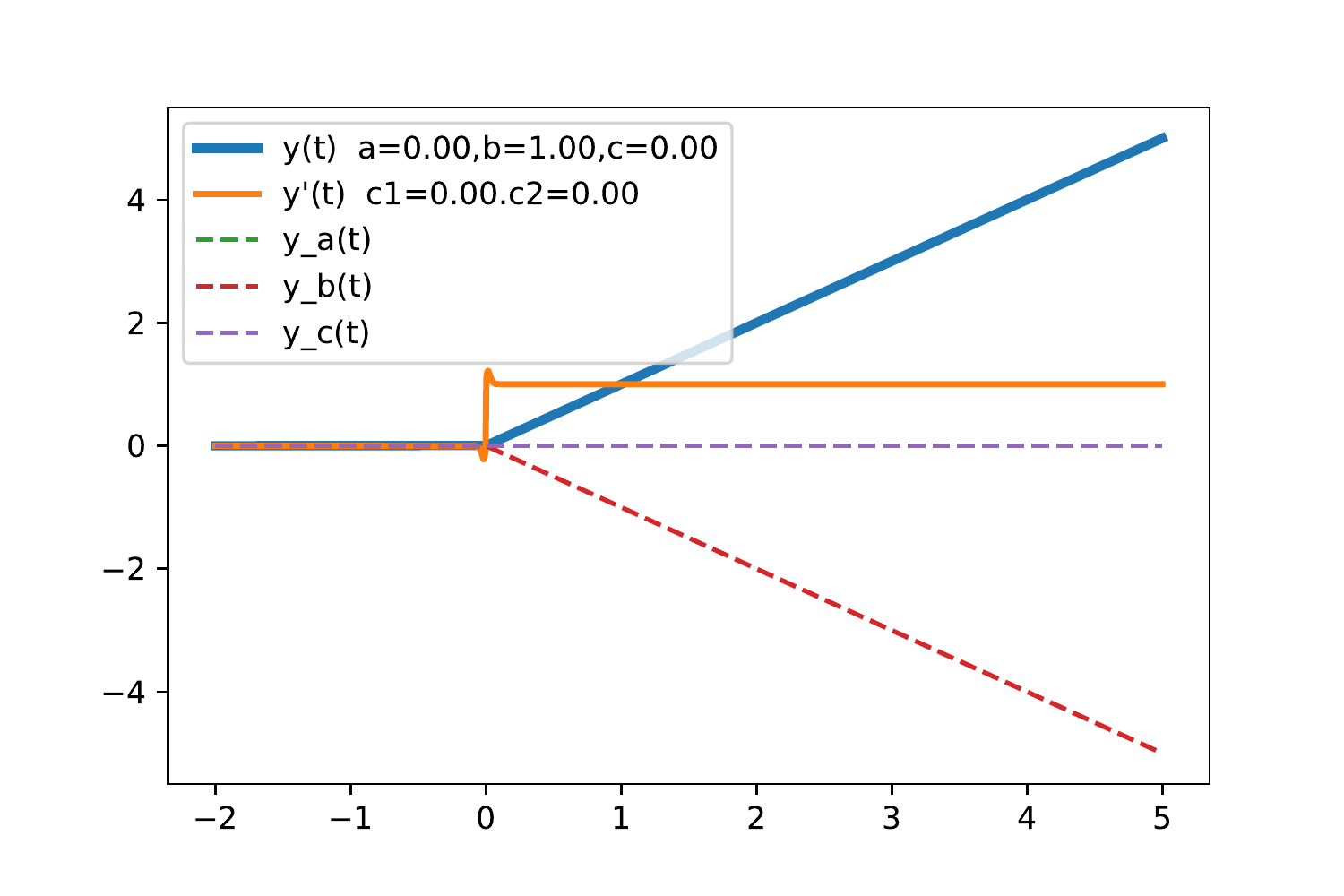}}
\subfigure{\includegraphics[width=0.3\linewidth]{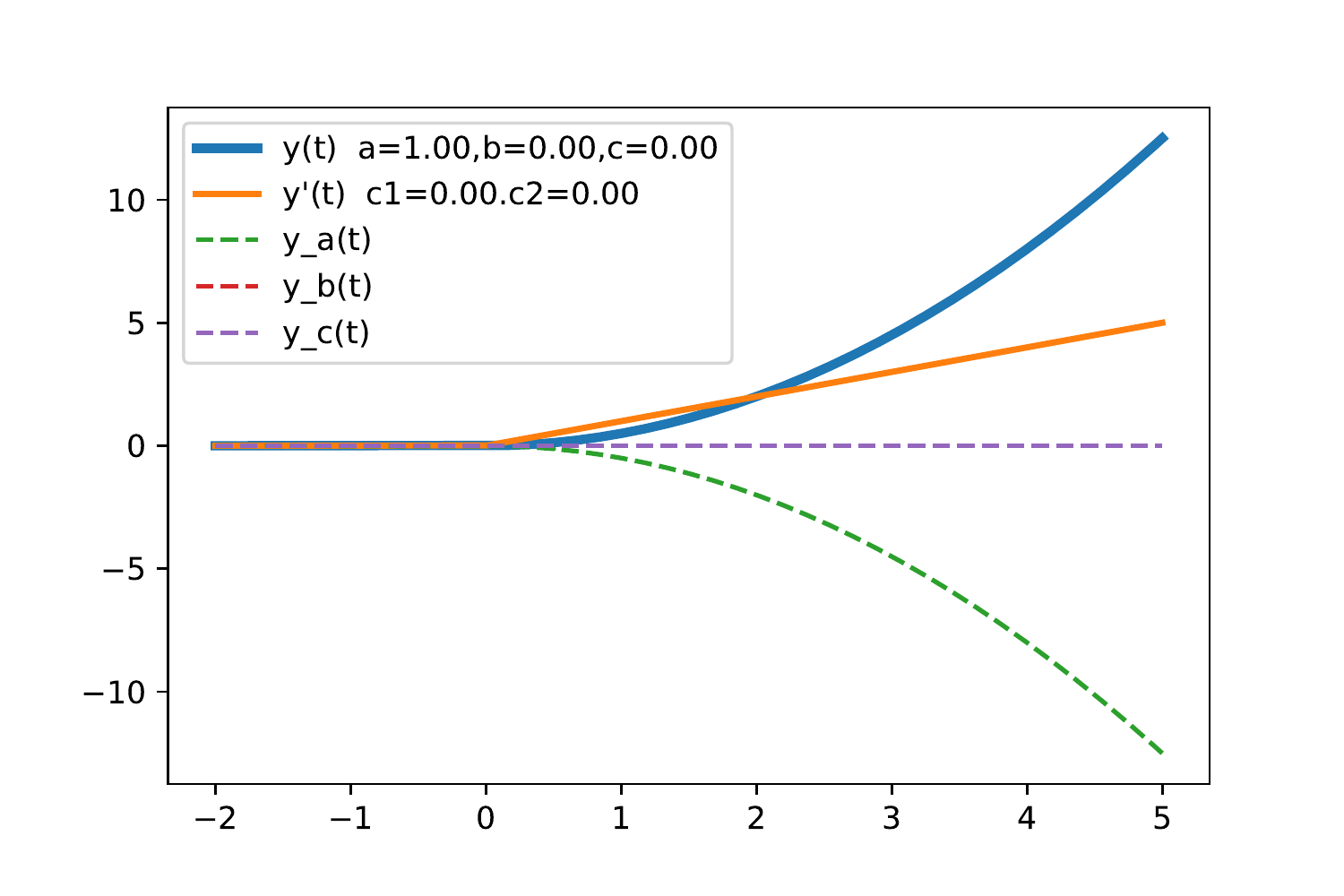}}
\subfigure{\includegraphics[width=0.3\linewidth]{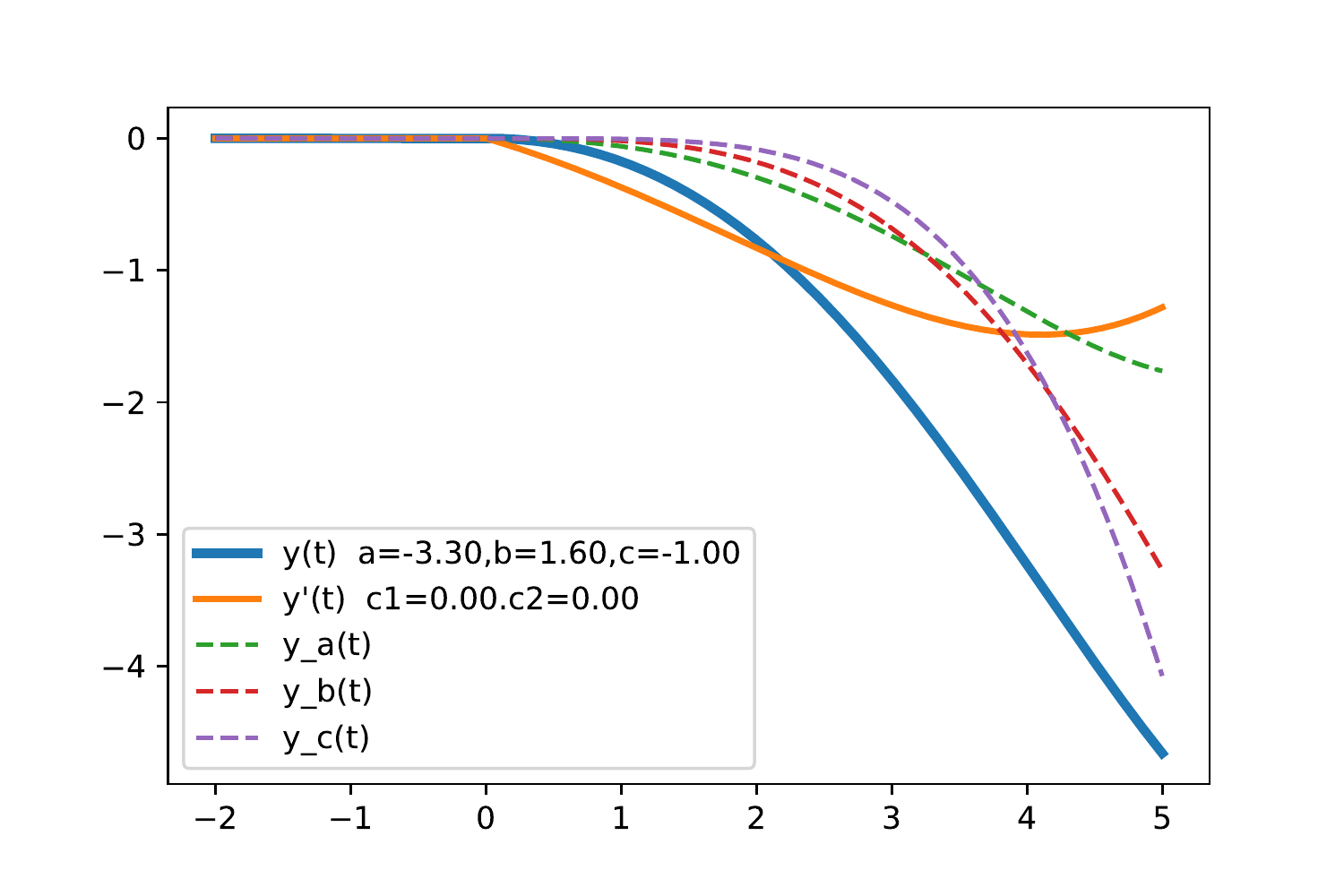}}
\subfigure{\includegraphics[width=0.3\linewidth]{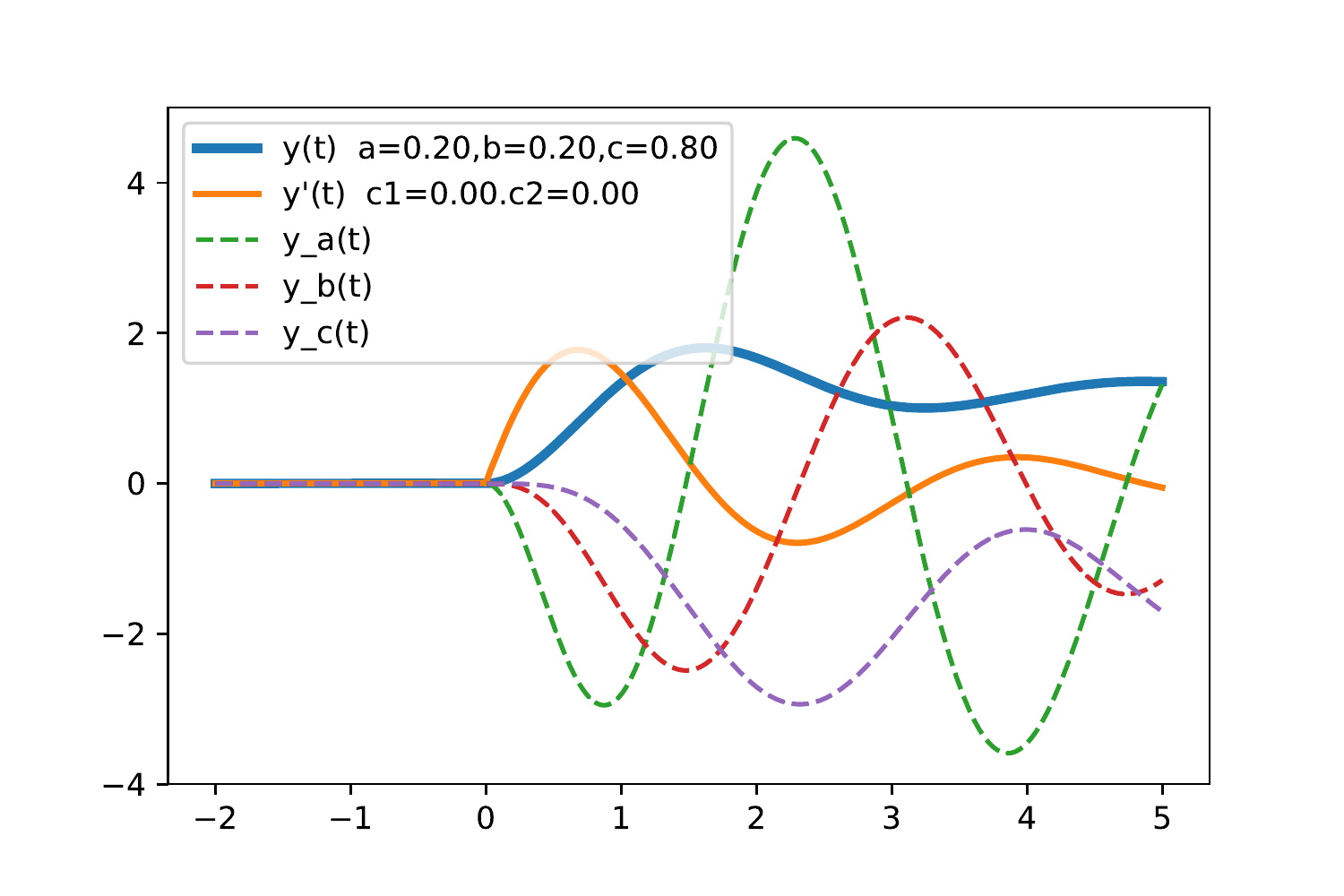}}
\subfigure{\includegraphics[width=0.3\linewidth]{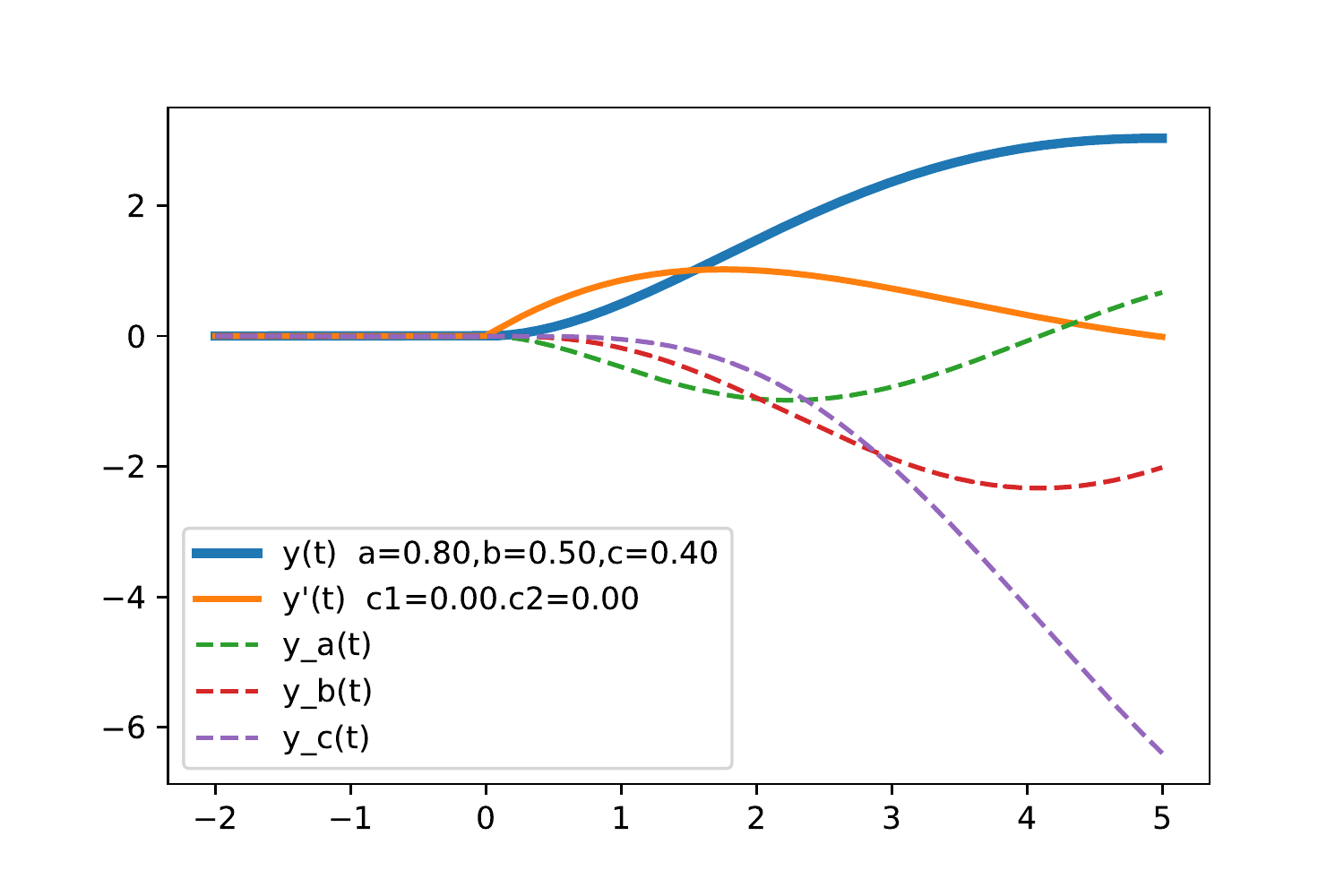}}
\subfigure{\includegraphics[width=0.3\linewidth]{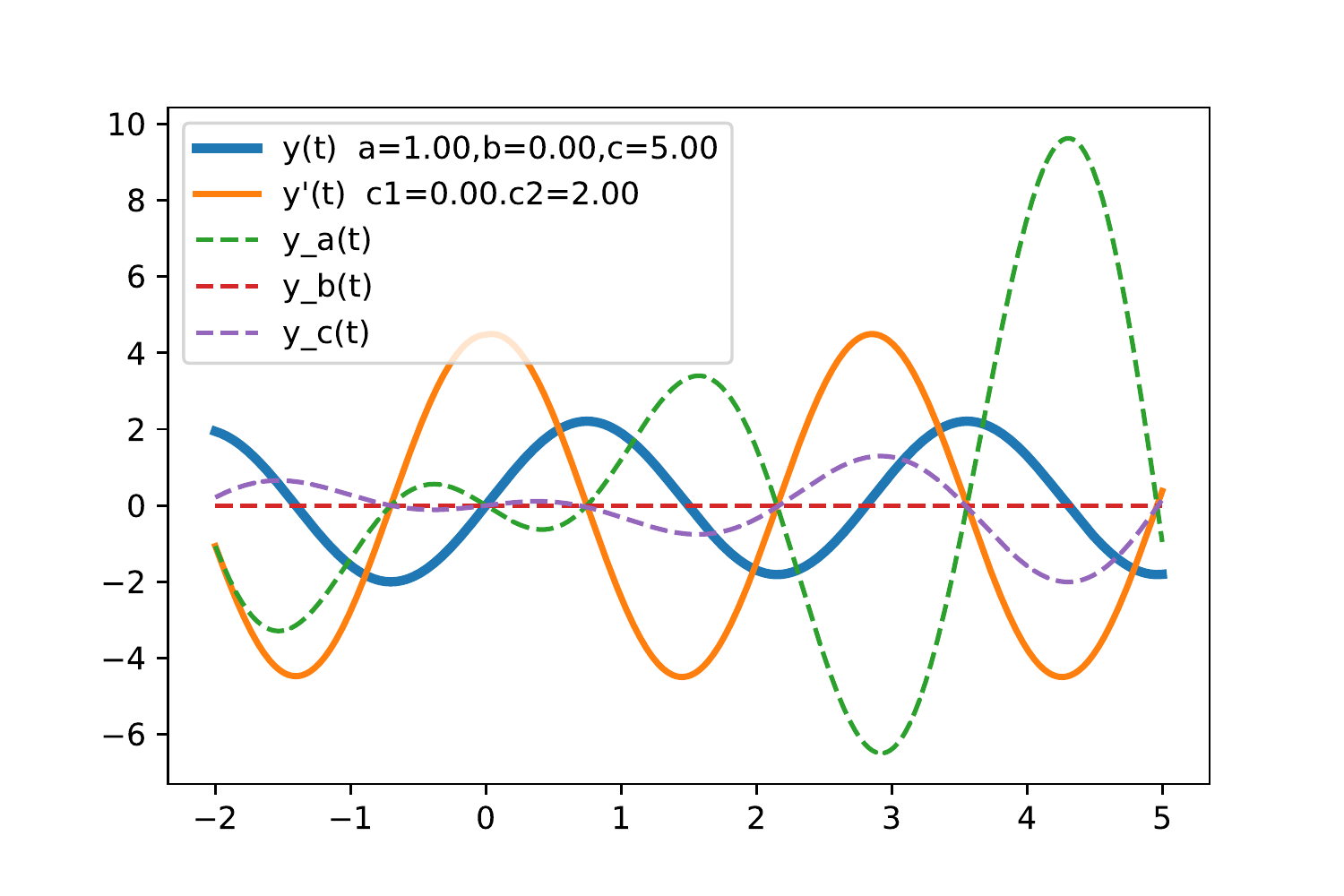}}
\caption{A sample set of DEU activation functions and their
 derivatives. The bold blue line is the activation function, and the orange solid line is its derivative with respect to $t$. The dashed lines are its derivatives with respect to $a$, $b$ and $c$. First and second on the top row from left are ReLU and
 ReQU. The bump in the derivative of ReLU is an artifact of approximating Dirac's delta.
\label{fig:example_fcns}
}
 
\end{figure*}

\subsubsection{Outward gravitation}
We need to let a DEU to jump out of a singular subspace when the network benefits from it. This way, no matter what the initialization has been and how the order of training samples has changed the geodesic path that an activation function follows to get the appropriate functional form, the activation function can recover from accidental falling into singularity points. To allow this, for a singular solution, we introduce a hypothetical non-singular differential equation that has the same initial conditions as the corresponding solution function at a reference point $t^*$.\footnote{The reference point can be $t^*=0$, or $t^*=t_0$ for some other choice of $t_0$. In our implementation, we set $t_0 = \frac{1}{N_{batch}}\sum_{i = 1}^{N_{batch}} t_i$ that is the center of  values in the training batch.}

For a parametric variable $p \in \{a,b,c\}$, \[\tilde{p} = \begin{cases} - \epsilon & - \epsilon < p <0 \\  \epsilon & 0 < p < \epsilon \\ p & \text{o/w.}\end{cases}\] is the closest value outside of the singular subspace ($\tilde{p}=p$ if the parameters are  not in a singularity region). If the value of any of the parameters $a$, $b$, or $c$ are in the singular subspace, we form a hypothetical differential equation $\tilde{a} y'' + \tilde{b} y' + \tilde{c} y = u(t)$, whose solution has the same initial condition  at the reference point $t^*$ as the current activation function (which is parameterized by $a$, $b$, $c$, $c_1$ and $c_2$). We use the derivative of the solution of this hypothetical differential equation for any of $a$, $b$, and $c$ that are in the singular subspace. It is important to certify that the initial conditions of the hypothetical ODE (i.e., the value of the function and its derivative with respect to input $t$) are the same as the initial conditions of the actual activation function (otherwise, the derivative with respect to parameters $a$, $b$, and $c$ would be from a function that behaves differently than the actual activation function). Let $y(t; a, b, c, c_1, c_2)$ be the solution of an activation function, then our goal is to choose $\tilde{c}_1$ and $\tilde{c}_2$ such that for $\tilde{y}(t; \tilde{a},  \tilde{b},  \tilde{c},  \tilde{c}_1,  \tilde{c}_2)$, we have: 

\begin{align} \label{eq_initial_val}
    & \tilde{y}(t^*; \tilde{a},  \tilde{b},  \tilde{c},  \tilde{c}_1,  \tilde{c}_2) = y(t^*; a, b, c, c_1, c_2)\\ \label{eq_initial_der}
    & \frac{\partial  \tilde{y}(t; \tilde{a},  \tilde{b},  \tilde{c},  \tilde{c}_1,  \tilde{c}_2)}{\partial t}\bigg\rvert_{t=t^*} = \frac{\partial  y(t; a, b, c, c_1, c_2)}{\partial t}\bigg\rvert_{t=t^*}
\end{align}

To find $\tilde{c}_1$ and $\tilde{c}_2$, we assign the values of $t^*$, $a$, $b$, $c$,  $c_1$,  $c_2$, $\tilde{a}$, $\tilde{b}$ and $\tilde{c}$ into equations (\ref{eq_initial_val}) and (\ref{eq_initial_der}), and  since $\tilde{y}$ is affine with respect to $\tilde{c_1}$ and $\tilde{c_2}$, the results forms a two-variable, two-equation system of linear equations that is solved in constant time (see appendix for the details). 


\eat{
\begin{algorithm} 
\caption{Differential equation network -- learning the parameters of a neuron} 
\label{alg1} 
\begin{algorithmic} 
 \INPUT 
 \STATE Mini-batch input from the previous layer $x =
 [x_1,x_2, \ldots,x_k]^T$, 
 \STATE Previous weights vector $w$, 
 \STATE Previous bias weight $w_0$, and 
 \STATE Backpropagated input errors for each sample $\delta_1, \ldots, \delta_k $
 \sForAll{$i$}{
 \quad $o_i$
}
 \STATE 1: Compute the derivatives
 \STATE $y \Leftarrow 1$
 \IF{$n < 0$}
 \STATE $X \Leftarrow 1 / x$
 \STATE $N \Leftarrow -n$
 \ELSE
 \STATE $X \Leftarrow x$
 \STATE $N \Leftarrow n$
 \ENDIF
 \WHILE{$N \neq 0$}
 \IF{$N$ is even}
 \STATE $X \Leftarrow X \times X$
 \STATE $N \Leftarrow N / 2$
 \ELSE[$N$ is odd]
 \STATE $y \Leftarrow y \times X$
 \STATE $N \Leftarrow N - 1$
 \ENDIF
 \ENDWHILE
\end{algorithmic}
\end{algorithm}}

\subsubsection{Implementation} Due to the individual activation functions associated parameters per neuron, each neuron can have a different activation. For example, one neuron could have a sinusoidal activation function while another has a quadratic form of ReLU. The direct way to implement our method would involve iterating over all neurons and computing the corresponding activation values individually, and would result in significantly higher latency than common activation functions such as ReLU. However, considering the ODE's closed form solutions, such computations are parallelizable on modern computing units like GPUs. We first compute a mask to assign each neuron to the subspace in which its current activation function resides. Next, we iterate over all possible functional subspaces, and compute in parallel the output of each neuron assuming that its activation function lies in the current subspace while ignoring the validity of the solutions or parameters. Finally, we use the subspace masks computed earlier to select the correct output of the neuron. A similar trick can be applied during the backward pass to compute the gradients in parallel using the masks computed during the forward pass. The pseudo-code is detailed in Algorithm \ref{alg:deu_alg}.

\begin{algorithm}
		\caption{Parallelized DEU}
		\label{alg:deu_alg}
		\begin{algorithmic}[1]
		    \Procedure{Deu}{$\text{input}$} 
			\State $\text{output} \leftarrow 0$
			\For {each singularity space S}
			\State $mask = \bm{1}_{[ n \in S]} \quad \forall n \in neurons$
			\If {$\text{mask} > 0$}
			\State $\text{output} \leftarrow \text{output} +  \text{mask}*f_S(\text{input})$

			\EndIf
			\EndFor
			\State Return $\text{output}$
			\EndProcedure
		\end{algorithmic}
\end{algorithm}

\subsection{Neural networks with DEUs are universal approximators}
Feedforward neural networks with monotonically-increasing activation functions are universal approximators
\cite{hornik1989multilayer,barron1993universal}. Networks with
radial basis activation functions that are bounded, increasing and then decreasing are also
shown to be universal approximators \cite{park1991universal}. A neural network with DEU is also a universal approximator (see Appendix for the proof). 

\subsubsection{A geometric interpretation}
The solution set of a differential equation forms a functional
manifold that is affine with respect to $c_1$ and $c_2$, but is
nonlinear in $a$, $b$, and $c$. Clearly, this manifold has a trivially
low dimensional representation in $\mathbb{R}^5$ (i.e.,
$\{a, b, c, c_1, c_2\}$). Gradient descent changes the
functionals in this low dimensional space, and the corresponding
functional on the solution manifold is used as the learned activation function. 
Figure~\ref{fig:manifold} attempts to visually explain
how, for example, a ReLU activation function transforms to a cosine
activation function. 

\begin{figure}[!tbh]
 \centering
 \includegraphics[width=.6\textwidth]{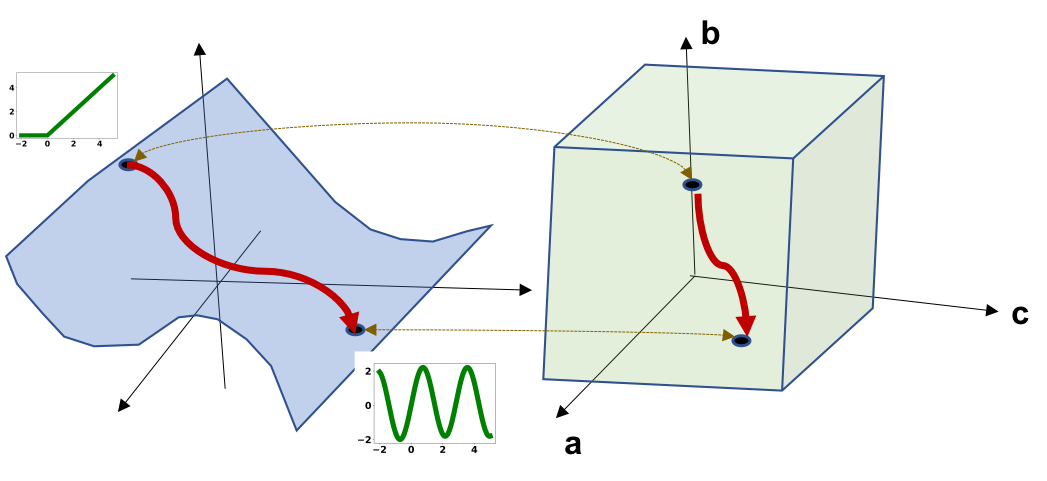}
 \caption{Left: the solutions of $ay''+by'+cy= u(t)$ lie on a
 manifold of functions. Right: every point on this
 manifold can be equivalently represented as a point on a
 $5$-dimensional space of $\{a,b,c,c_1,c_2\}$ (three shown here). The red arrow shows a path that an initialized function
 takes to
 gradually transform itself to a different one.}
 \label{fig:manifold}
\end{figure}

In appendix (toy datasets), we empirically show that only one DEU neuron that is initialized with ReLU ($a=c=c_1=c_2=0,~b=1$) can transform itself to a sine function. In this scenario, the learned differential equation coefficients $\{a, b, c\}$
and initial condition coefficients $\{c_1, c2\}$ perfectly represented a sine function after training. In contrast, the learned model by
ordinary fixed activation FFNNs were much less accurate with significantly larger networks. Figure~\ref{fig:var_coef} (in the appendix) visually shows how changing a parameter changes the activation function's behavior.

\subsubsection{Reduction to common activation functions}

Two common activation functions are the sigmoid $\sigma(t)= \frac{1}{1+e^{-t}}$, and the rectified linear unit
$\text{ReLU}(t)= \max(0,t)$. The sigmoid function is a smooth approximation of
the Heavyside step
function, and ReLU can be approximated by integrating
sigmoid of $s*t$ for a large enough $s$: $\max(0,t) \approx \int_{-\infty}^t \frac{1}{1+e^{-s z}}dz=\log(1 + e^{s t})/s$. \eat{(Figure~\ref{fig:sig_int})}Equivalently, $y(t)=\log(1 + e^t) +c_1\approx \text{ReLU}(t)+c_1$ will be a solution of the following first order linear
differential equation: $y'(t) = \frac{1}{1+e^{-st}}\approx u(t)$

\eat{
\begin{figure}[tbh]
 \centering
 \includegraphics[width=.35\textwidth]{charts/relu_sig_int.pdf}
 \caption{Sigmoid vs. step function \& integration of sigmoid vs. ReLU}
 \label{fig:sig_int}
\end{figure}
}

We can set the core activation function $g(t)$ to $\sigma(t) = \frac{1}{1+e^{-s t}}$, or to the step
function $u(t)$. For $g(t)=\sigma(t)$ when $a\neq 0$, the solutions of the differential 
equation will involve the Gauss hypergeometric
and $Li_2$ functions, which are expensive to evaluate. Fortunately, if we set the right hand side of the ODE to $u(t)$, then the particular solutions will only involve
common functions such as linear, exponential, quadratic, trigonometric, and hyperbolic functions.

In practice, if the learning algorithm decides that $a$ and $b$ should be zero, we use
$g(t)=\frac{1}{1+e^{-st}}$ (i.e. $y(t) = \frac{1}{c*(1+e^{-st})}$). Otherwise, we use the step function to avoid complex-valued solutions that involve special mathematical functions. 
With these conditions in place, if $a=0$, $b=0$, and $c=1$, we recover the
sigmoid function; if $a=0$, $b=1$, and $c=0$, we recover the
ReLU function; if $a=1$, $b=0$, and $c=0$, we obtain a parametric rectified
quadratic form $y = \text{ReLU}(t)^2+c_1t+c_2$ (similar to parametric ReLU \cite{he2015delving,xu2015empirical}), which is the solution of
$y''(t) = u(t)$. When $b^2-4ac<0$, we observe an oscillatory
behaviour. Depending on the sign of $b$, this can be decaying or exploding, but when $b=0$, we observe a purely oscillating activation function.

The above-mentioned cases are only a few examples of solutions that could
be chosen. The point to emphasize here is that an extensive range of
functions can be generated by simply varying these few parameters (Figure~\ref{fig:example_fcns} illustrates several examples).

\section{Related Work}
Recently, \citeauthor{chen&al18}~(\citeyear{chen&al18}) propose using ODE solvers in neural networks. However, the nature of our work is very different from Neural ODE. While that work maps iterative parts of the network to an ODE and uses external solver to find the solution and estimating the gradient, here we use the flexible functions that are the solutions of ODE as activation function. At the training time there is no ODE solver involved, but only functional forms of the solutions.

Oscillatory neurons are generally believed to be important for information processing in animal brains.
Efforts have been made to explore the usefulness of periodic oscillations in neural networks since the 1990s, especially for medical applications \cite{MinamiNT99}. However overall their applicability has not yet been appreciated \cite{sopena1999neural}.
In recent times, however, researchers have begun re-exploring the potential of periodic functions as activations~\cite{parascandolo2016taming}, and demonstrated their potentials for probabilistic time-series forecasting~\cite{HatalisK17} and lane departure prediction~\cite{TanW17}. Furthermore recent work has demonstrated how pattern recognition can be achieved on physical oscillatory circuits~\cite{VelichkoBB19}.

For vision, speech and other applications applications on mobile, or other resource-constrained devices, research has been ongoing to make compact networks. ProjectionNet~\cite{ravi2017projectionnet} and MobileNet~\cite{howard2017mobilenets} are both examples of methods that use compact DNN representations with the goal of on-device applications. In ProjectionNet, a compact \textit{projection} network is trained in parallel to the primary network, and is used for the on-device network tasks. MobileNet, on the other hand, proposes a streamlined architecture in order to achieve network compactness. In these approaches, the network compactness is achieved at the expense of performance.
We propose a different method for learning compact, powerful, stand-alone networks: we allow each neuron to learn its individual activation function enabling a compact neural network to achieve higher performance.


\section{Results and Discussion}

We have conducted several experiments to evaluate the performance and compactness of DEU networks.

\begin{table}[t]
\centering
\caption{\label{tab:exp_img} Test accuracy of different models on the MNIST and Fashion-MNIST image classification task.}
\begin{tabular}{l r || r r}
\hline
\toprule
Model & Size & MNIST & Fashion-MNIST \\

\toprule
MLP-ReLU & 1411k & 98.1 &  89.0\\ 
CNN-ReLU & 30k &99.2 & 90.2 \\
\midrule
MLP-SELU & 1293k & 95.5  & 87.5  \\
CNN-SELU & 21k & 98.8 &  89.6   \\
\midrule
MLP-PReLU & 1293k & 97.4    & 88.7       \\
CNN-PRelu & 21k & 98.9 &  89.6  \\
\midrule
MLP-DEU & 1292k & 98.3 &   89.8 \\
CNN-DEU & 21k & 99.2 & \textbf{91.5}  \\
\midrule
MLP-Maxout & 2043k & 98.5 &   89.4 \\
CNN-Maxout & 26k & \textbf{99.4} & 91.3  \\
\bottomrule
\end{tabular}


\end{table}




\begin{table*}
\caption{\label{tab:exp_cifar} Test accuracy using different ResNet architectures and activation functions on the CIFAR-10 image classification task.}
\centering
\begin{tabular}{l r ||r r r r r r}
\toprule
Architecture & Size & ReLU & PReLU&  SELU & Swish & Maxout & DEU  \\
\toprule
ResNet-18 &  11174k & 91.25 & 92.1 & 92.2& 91.9 & \textbf{92.5}& \textbf{92.5} \\
Preact-ResNet & 11170k & 92.1 & 92.2 & 92.3 & 92.0 & \textbf{92.4} & 92.3  \\
ResNet-Stunted & 678k & 89.3 & 89.4 & 90.5& 90.1 & \textbf{91.1}& 90.7 \\
\bottomrule
\end{tabular}

\end{table*}


\subsection{Classification}
We evaluate DEU on different models considering the classification performance and model size. We first use MNIST and Fashion-MNIST as our datasets to assess the behavior of DEUs with respect to the commonly used ReLU activation function, as well as Maxout and SELU. 
The neural network is a 2-layer MLP with 1024 and 512 dimensional hidden layers. While the CNN used is a 2-layer model made by stacking 32 and 16 dimensional convolutional filters atop one another followed by average pooling.
DEUs are competitive or better than normal networks for these tasks while having substantially smaller number of parameters (see Table \ref{tab:exp_img}).

Next we perform a more direct comparison of the effect of DEU on classification performance against ReLU, PReLU~\cite{he2015delving}, Maxout~\cite{goodfellow2013maxout}, and Swish~\cite{ramachandran2017searching} activation functions on the CIFAR-10 dataset. PReLU is similar to ReLU with a parametric leakage and Swish has the form of $f(x) = x*\textit{sigmoid}(\beta x)$ with a learnable parameter $\beta$.

For these experiments, we keep the network architecture fixed to ResNet-18~\cite{he2016deep} and use the hyperparameter settings as in He et al.~(\citeyear{he2016deep}). 
We further show that this improvement persists across other model designs. First we use a preactivation ResNet~\cite{he16preact}, which is a ResNet-like architecture with a slightly smaller size. Second, to assess suitability for reducing the model size, we experiment with a stunted ResNet-18. The stunted model is created by removing all 256 and 512 dimensional filters from the standard ResNet-18 model. The result of this comparison (presented in Table~\ref{tab:exp_cifar}) indicates that DEUs constantly work better than ReLU, PReLU, SELU, and Swish. Although Maxout is slightly better than DEU, it is using much more parameters (13272k vs 11174k for Resnet,  13267k vs 11174k for Preact-ResNet, and 801k vs 678k for ResNet-Stunted), which makes the comparison biased toward Maxout. \\
\subsubsection{Convergence comparison on MNIST} 
Figure~\ref{fig:mnist_comp} shows the classification error on MNIST across different activations as the training progresses. It is clear that DEUs are better at almost all steps. Also as one might expect, they are better in the initial epochs due to the greater expressiveness of the activation function.\\

\begin{figure}[!tbh]
\centering  
\includegraphics[width=.6\columnwidth]{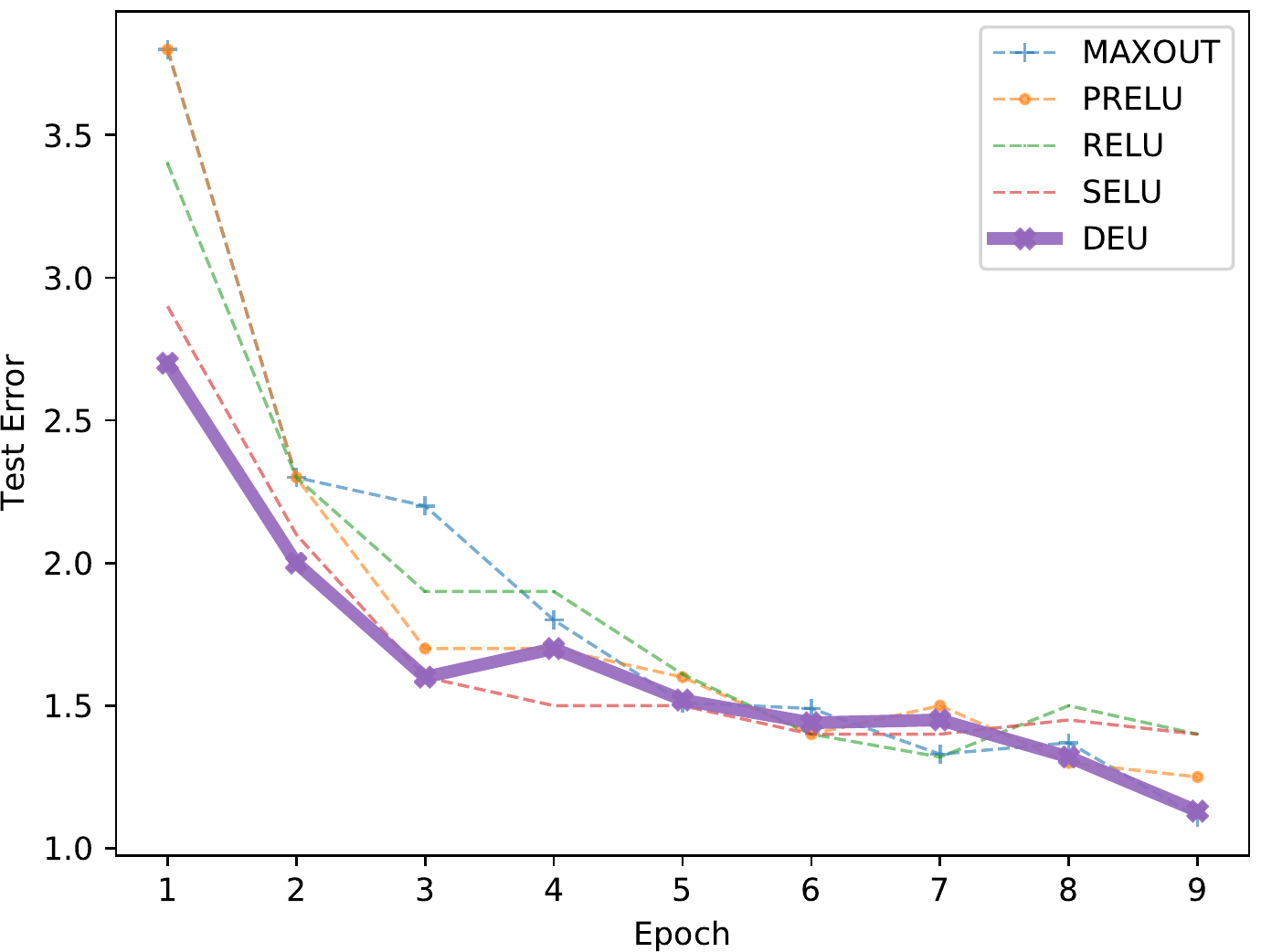}
\caption{Convergence comparison of different activation function on MNIST.}
\label{fig:mnist_comp}
\end{figure}



\subsubsection{Computational cost}
Prima-facie our method seems extremely compute-intensive. However as described earlier, with judicious design one can parallelize the overall model to be quite fast. In the worst case, our method will be $|S|$ times slower than the same architecture with standard activation like ReLU, where $|S|$ is the number of singular subspace solution of the ODE. Nevertheless, in practice all subspaces are unlikely to occur together. In most of our experiments on real data, we observed that three to five different solutions appear.
Furthermore, we evaluate computation time of Resnet-18 models with DEUs and ReLUs on CIFAR-10 using a batch size of 128 on a Tesla K40c GPU. The model with DEUs takes 640ms total for combined forward and backward pass per batch, while the one with ReLUs requires 180ms per step. 
During prediction on test data, the time taken per batch is 111ms for the DEU-based model and 65ms for ReLU-based model.

\subsection{Regression}
We compare the performance of neural networks with one hidden-layer of one, two, four, eight, and 16 neurons with DEU, ReLU, LeakyReLU (LReLU), SELU, Swish and Maxout activation functions applied to a standard diabetes regression dataset.\footnote{(https://www4.stat.ncsu.edu/~boos/var.select/diabetes.html)}  We use 3-fold cross validation and report the average performance.

Figure~\ref{fig:reg_comp} shows that the neural networks with DEUs achieve specifically better performance with more compact networks.  We see that other activation functions do not surpass the single DEU neuron performance until they are eight or more neurons.

\begin{table}[h]
\caption{Evaluation summary (in RSE and CORR) of different activations on different datasets: {\bf Traffic.} A collection of 48 months (2015-2016) hourly data
from the California Department of Transportation (6 and 12 hours look ahead prediction); \textbf{Solar-Energy}. Solar power production records in the year of 2006 sampled every 10 minutes from 137 PV plants in Alabama State (6 and 12 hours look ahead prediction); \textbf{Electricity}. The electricity consumption in kWh  recorded
every 15 minutes from 2012 to 2014(12 hours look ahead prediction) \cite{LaiCYL17}).}
\centering
\small{
\begin{tabular}{ l||l c c c c c }
\toprule
Func.               &      & \multicolumn{2}{c}{Traffic} & \multicolumn{2}{|c|}{Solar output} & Elect. \\
\midrule
                       &      & 6            & 12           & 6               & 12             &  12           \\
\midrule
                     
\multirow{2}{*}{DEU}   & RSE  & 0.487       & {\bf 0.500}       &  0.269          & {\bf 0.327}         & {\bf 0.100}      \\
                       & CORR & 0.870        & {\bf 0.863}        &  0.965           & {\bf 0.947}          & {\bf 0.912}       \\
                       \hline
\multirow{2}{*}{ReLU}  & RSE  & 0.499       & 0.520       & 0.270          & 0.433         & 0.104      \\
                       & CORR & 0.869        & 0.851        &  0.965           & 0.906          & 0.900       \\
                       \hline
\multirow{2}{*}{Swish} & RSE  & {\bf 0.483}       & 0.505       & 0.270          & 0.329          & 0.104      \\
                       & CORR & {\bf 0.872}       & 0.862        & 0.965          & 0.944          & 0.908   \\
                       \hline
\multirow{2}{*}{Maxout} & RSE  &  0.493      & 0.501       & {\bf 0.265}          & 0.328          & 0.107      \\
                       & CORR &  0.863       & 0.868        & {\bf 0.967}           & 0.945          & 0.911   \\
\bottomrule
\end{tabular}
}
\label{tab:regression_comp}
\end{table}

\begin{figure}[t]
\centering  
\includegraphics[width=.8\columnwidth]{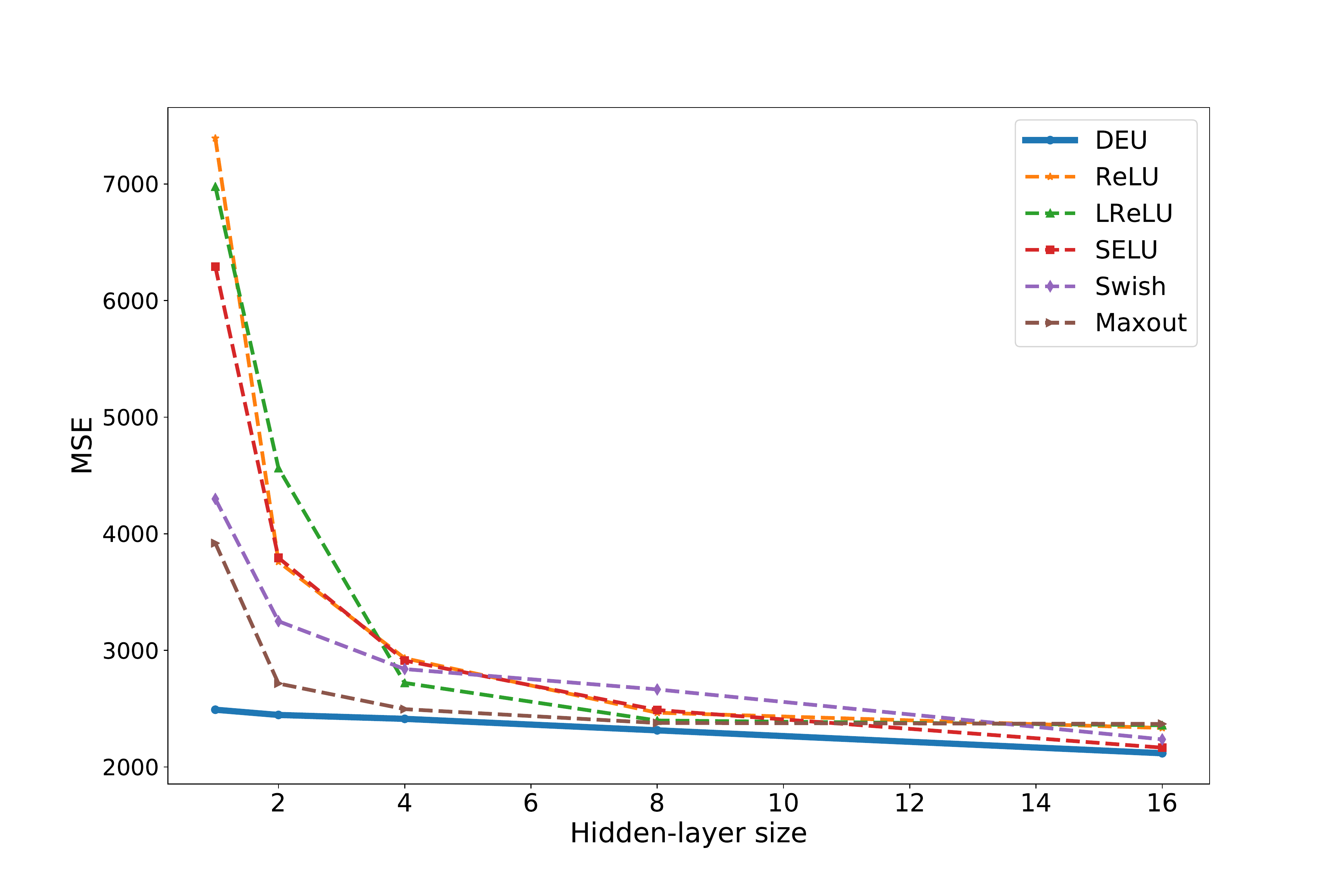}
\caption{Convergence comparison of Diabetes.}
\label{fig:reg_comp}
\end{figure}

We test our activation function on multiple regression tasks involving timeseries forecasting. For these tasks, we use the LSTNet model~\cite{LaiCYL17} and evaluate the functions based on the root relative squared error (RSE, lower is better) and correlation (CORR, higher is better) metrics. The results are presented in Table \ref{tab:regression_comp}. DEU gives improvement in all cases and substantial improvements in some datasets.

\section{Conclusion}

In this paper we introduce  differential equation units (DEUs), as  novel activation functions based on the solutions of second-order ordinary differential equations (ODEs). DEUs can adapt their function form based on the features of data during training by learning the parameters of ODEs using gradient descent. We have showcased the ability of neural networks with DEUs to learn complicated concepts with a compact network representation. We have demonstrated DEUs' potential to outperform conventional activation function across a number of tasks, and with a reduced network size. Modern DNNs achieve performance gains in large by increasing the size of the network, which is not a sustainable trend. In response, we believe that this line of research can open future directions to explore more complex activation functions such as using the solutions of partial differential equations in order to compactly represent complex functions.


\bibliography{all}
\bibliographystyle{iclr2019_conference}

\clearpage

\appendix

\section{Neural networks with DEUs are Universal Approximators}\label{app:universal}

\begin{lemma}\label{lemma_osci}
If $b^2-4ac<0$, the solutions of $a y''(t) + b y'(t) + c y(t) =u(t)$
will oscillate with frequency $\omega =\frac{\sqrt{4ac-b^2}}{2a}$, and in particular, if $b=0$, then $\omega =\sqrt{\frac{c}{a}}$.
\end{lemma}
\begin{proof}
If $b^2-4ac<0$, the roots of the characteristic equation of the ODE will be $\frac{-b\rpm i \sqrt{4ac-b^2}}{2a}$, where $i=\sqrt{-1}$. By substituting the solutions in Euler's formula the resulting sine and cosine functions will have frequency $\omega =\frac{\sqrt{4ac-b^2}}{2a}$. In particular, if $b=0$, we'll have $\omega =\sqrt{\frac{c}{a}}$. 
\end{proof}

Particularly, the solution of the ODE $a y''(t) + c y(t) =u(t)$ for $ac>0$ is $y(t) = \sin( \sqrt{\frac {c}{a}}t) c_2+\cos(\sqrt{\frac {c}{a}}t) c_1-\frac {u(t) }{c} \left( \cos(\sqrt{\frac {c}{a}}t) -1 \right)$. We will use this in the proof of the following theorem:

\begin{theorem}
A neural network with DEU is a universal approximator.
\end{theorem}

\begin{proof}
By Lemma~\ref{lemma_osci}, for $b=0$ and $c=1$, the oscillation frequency of the activation function of a neuron will be $\omega =1/\sqrt{a}$. Therefore, by varying $a \in \mathbb{R}^+$, a neuron activation function can generate all oscillation frequencies. Now, consider a neural network with one hidden layer of differential equation neurons and one linear output layer. For a target function $h(t)$, we take the Fourier transformation of $h(t)$, and find the set of present frequencies $\Omega$. For each present frequency $\omega_i \in \Omega$, we add two differential equation neuron $i_1$ and $i_2$ to the hidden layer with $c=1$, $b=0$ and $a=1/\omega_i^2$. Let $\beta_i$ and $\gamma_i$ be the corresponding coefficient of the sine and cosine functions from the Fourier transformation. Then, we let $c_{1i_1} = \gamma_i, c_{2i_1} = \beta_i$ and $c_{1i_2} = c_{2i_2} = 0 $, and corresponding weights $w_{i_1}$ and $w_{i_2}$ from neurons $i_1$ and $i_2$ to the output linear layer to be equal to $1.0$ and $-1.0$, respectively. This way, the two neurons will cancel the particular solution term of the function, and we'll have: $y_{i_1}(t)-y_{i_2}(t) = \beta_i \sin(\omega_i t) + \gamma_i \cos(\omega_i t)$. By construction, after adding all frequencies from the Fourier transformation of $h(t)$, the neural network will output $\sum_{i\in \Omega} \beta_i \sin(\omega_i t) + \gamma_i \cos(\omega_i t) = h(t)$.
\end{proof}
There are multiple other ways to prove this theorem. For example, since the Sigmoid, ReLU, and ReLU-like functions are among the solutions
of differential equations, we can directly apply results of the Stone-Weierstrass
approximation theorem \cite{de1959stone}.

\section{The Euler-Lagrange Equation in Calculus of Variations}
\label{calc_of_var}

As practiced in calculus of variations, a functional $\mathbf{J}$ is a function of functions, and it is usually expressed as a
definite integral over a mapping $L$ from the function of interest $y$ and its derivatives at point $t$ to a real number: $\mathbf{J}[y] =
\int_{t \in \mathcal{T}}L(t,y,y',y'',\ldots,y^{(n)})dt$. $y(t)$ is the function of interest, and $y',y'',\ldots,y^{(n)}$ are its derivatives up to the $n$th order. 

Intuitively, $L$ is a combining function, which reflects the cost (or the reward) of selecting the function $y$ at point $t$, and the integral sums up the overall costs for all $t \in \mathcal{T}$. Therefore, the functional $\mathbf{J}:(\mathbb{R}\rightarrow\mathbb{R})\rightarrow\mathbb{R}$ can represent the cost of choosing a specific input function. The minimum of the functional is a function $y^*(t) = \arg\min_y
\mathbf{J}(y)$ that incurs the least cost. Among other methods, $y^*(t)$ can be found by solving a
corresponding differential equation obtained by the Euler-Lagrange
equation~\cite{gel1963variatsionnoeischislenie,gelfand2000calculus}. In particular, the extrema to $\mathbf{J}[y] =
\int_{t_1}^{t_2} L(t,y,y',y'')dt$ are the same as the solutions of the following differential equation: 
\begin{align}
\frac{\partial L}{\partial y}-\frac{d}{dt}\frac{\partial L}{\partial y'}+\frac{d^2}{dt^2}\frac{\partial L}{\partial y''}=0
\end{align}

\eat{

\begin{proposition}\label{euler_lagrange_connection}
For any values of $a,b,c \in \mathbb{R}$, solving the following differential equation:
$a y''(t) + b y'(t) + c y(t) = u(t)$ is equivalent to extremizing a functional $\mathbf{J}[y] =
\int L(t,y,y',y'')dt$, where $L(t,y,y',y'') = \frac{1}{2} ([y(t),y'(t),y''(t)] \mathbf{A} [y(t),y'(t),y''(t)]^T + (y(t)-\zeta u(t))^2)$ with $A=A^T\in \mathbb{R}^{3\times 3}$ a symetric matrix and $\zeta\in \mathbb{R^+}$ a positive constant.
\end{proposition}
\begin{proof}
Let $\mathbf{A} = \begin{bmatrix}
A_1 & A_2 & A_3\\
A_2 & A_4 & 0\\
A_3 & 0 & 0\\
\end{bmatrix}$  By expansion, $L(t,y,y',y'')= \frac{1}{2}A_1 y^2 + A_2 yy' + A_3 yy'' + \frac{1}{2}A_4 y'^2 + \frac{1}{2} y^2 -  \zeta u(t) y + \frac{1}{2}\zeta^2 u(t)^2$.
We have:
\begin{align}
\frac{\partial L}{\partial y} = (A_1+1) y + A_2 y' + A_3 y'' - \zeta u(t) \nonumber\\
\frac{\partial L}{\partial y'} = A_2 y+ A_4 y' \Rightarrow \frac{d}{dt}\frac{\partial L}{\partial y'}= A_2 y'+ A_4 y''\nonumber \\
\frac{\partial L}{\partial y''} = A_3 y \Rightarrow \frac{d^2}{dt^2}\frac{\partial L}{\partial y''}= A_3 y'' \nonumber
\end{align}
\begin{align}
\Rightarrow \frac{\partial L}{\partial y} - \frac{d}{dt}\frac{\partial L}{\partial y'} + \frac{d^2}{dt^2}\frac{\partial L}{\partial y''} = 
 = (A_1+1) y + A_2 y' + A_3 y'' - \zeta u(t) \nonumber\\
= A_2 y'+ A_4 y''\nonumber \\
= A_3 y'' \nonumber
\end{align}

\end{proof}}
\section{Outward Gravitation, Finding $\tilde{c}_1$ and $\tilde{c}_2$}

Let $y(t; a, b, c, c_1, c_2)$ be the current activation function solution, and $\tilde{a}$,  $\tilde{b}$,  $\tilde{c}$ be values outside the singularity subspace. We need to select $\tilde{c}_1$ and $\tilde{c}_2$ such that:

\begin{align} \label{eq_initial_val}
    & \tilde{y}(t^*; \tilde{a},  \tilde{b},  \tilde{c},  \tilde{c}_1,  \tilde{c}_2) = y(t^*; a, b, c, c_1, c_2)\nonumber\\ 
    & \frac{\partial  \tilde{y}(t; \tilde{a},  \tilde{b},  \tilde{c},  \tilde{c}_1,  \tilde{c}_2)}{\partial t}\bigg\rvert_{t=t^*} = \frac{\partial  y(t; a, b, c, c_1, c_2)}{\partial t}\bigg\rvert_{t=t^*}\nonumber
\end{align}

Since $y = f(t;a,b,c) + c_1 f_1(t;a,b,c)+ c_2 f_2(t;a,b,c)$ and $\tilde{y} = \tilde{f}(t;\tilde{a},\tilde{b},\tilde{c}) + \tilde{c_1} \tilde{f}_1(t;\tilde{a},\tilde{b},\tilde{c})+ \tilde{c_2} \tilde{f}_2(t;\tilde{a},\tilde{b},\tilde{c})$, we will need:

\begin{align} 
    & \tilde{f}(t^*;\tilde{a},\tilde{b},\tilde{c}) + \tilde{c_1} \tilde{f}_1(t^*;\tilde{a},\tilde{b},\tilde{c})+ \tilde{c_2} \tilde{f}_2(t^*;\tilde{a},\tilde{b},\tilde{c}) = \nonumber\\
    & \quad \quad f(t^*;a,b,c) + c_1 f_1(t^*;a,b,c)+  c_2 f_2(t^*;a,b,c)\nonumber \\
    & \frac{\partial  \tilde{f}(t;\tilde{a},\tilde{b},\tilde{c})}{\partial t}\bigg\rvert_{t=t^*} + \tilde{c_1} \frac{\partial  \tilde{f}_1(t;\tilde{a},\tilde{b},\tilde{c})}{\partial t}\bigg\rvert_{t=t^*} 
    + \tilde{c_2} \frac{\partial  \tilde{f}_2(t;\tilde{a},\tilde{b},\tilde{c})}{\partial t}\bigg\rvert_{t=t^*} = \nonumber\\
    & \quad \quad \frac{\partial  f(t;a,b,c)}{\partial t}\bigg\rvert_{t=t^*} + c_1 \frac{\partial  f_1(t;a,b,c)}{\partial t}\bigg\rvert_{t=t^*} + c_2 \frac{\partial  f_2(t;a,b,c)}{\partial t}\bigg\rvert_{t=t^*}\nonumber
\end{align}

Let $a_{11} = \tilde{f}_1(t^*;\tilde{a},\tilde{b},\tilde{c})$, $a_{12} = \tilde{f}_2(t^*;\tilde{a},\tilde{b},\tilde{c})$, $a_{21} = \frac{\partial  \tilde{f}_1(t;\tilde{a},\tilde{b},\tilde{c})}{\partial t}\bigg\rvert_{t=t^*}$, $a_{22}=\frac{\partial  \tilde{f}_2(t;\tilde{a},\tilde{b},\tilde{c})}{\partial t}\bigg\rvert_{t=t^*}$, $b_{11} = f(t^*;a,b,c) + c_1 f_1(t^*;a,b,c)+ c_2 f_2(t^*;a,b,c) - \tilde{f}(t^*;\tilde{a},\tilde{b},\tilde{c})$, $b_{12} = \frac{\partial  f(t;a,b,c)}{\partial t}\bigg\rvert_{t=t^*} + c_1 \frac{\partial  f_1(t;a,b,c)}{\partial t}\bigg\rvert_{t=t^*} + c_2 \frac{\partial  f_2(t;a,b,c)}{\partial t}\bigg\rvert_{t=t^*} - \frac{\partial  \tilde{f}(t;\tilde{a},\tilde{b},\tilde{c})}{\partial t}\bigg\rvert_{t=t^*}$, $A = \begin{bmatrix}a_{11}&a_{12}\\a_{21}&a_{22}\end{bmatrix}$, $B = \begin{bmatrix}b_{11} \\ b_{12}\end{bmatrix}$, then by substituting $t^*$, $a$, $b$, $c$, $c_1$, $c_2$ $\tilde{a}$,  $\tilde{b}$ and  $\tilde{c}$, we'll have: $\begin{bmatrix}\tilde{c}_{1} \\ \tilde{c}_{2}\end{bmatrix}= (A^TA + \lambda I)^{-1} A^T B$, where $\lambda=1e-9$ and $I$ is the two by two identity matrix. ($\lambda I$ is added in case $A^TA $ is ill-conditioned.)

\section{Example Spectra of Possible Activation Functions}\label{app:extra}
\begin{figure}[!tbh]
\centering 
\subfigure[ ]{\includegraphics[width=0.32\linewidth]{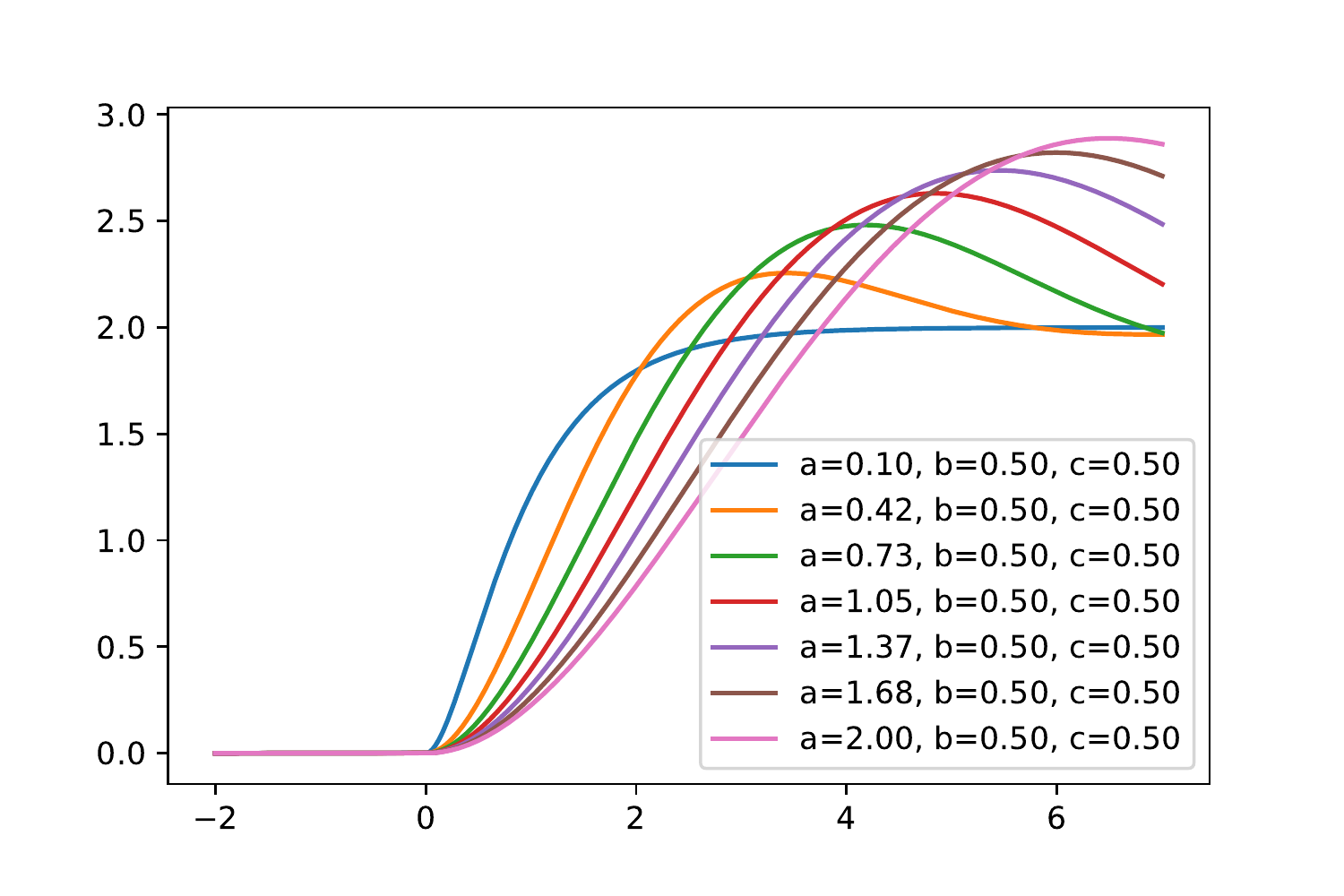}}\label{fig:var_a}
\subfigure[ ]{\includegraphics[width=0.32\linewidth]{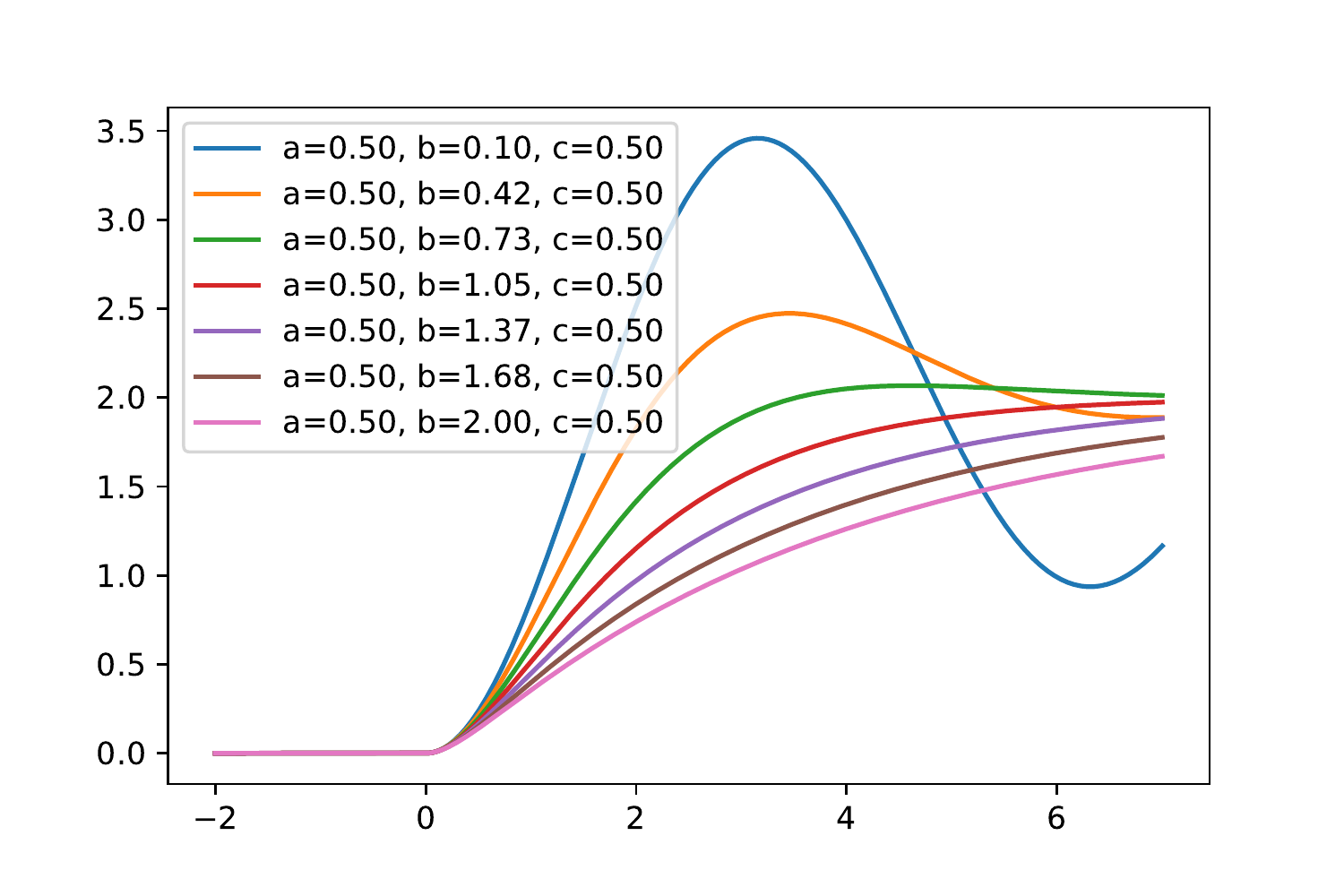}}\label{fig:var_b}
\subfigure[ ]{\includegraphics[width=0.32\linewidth]{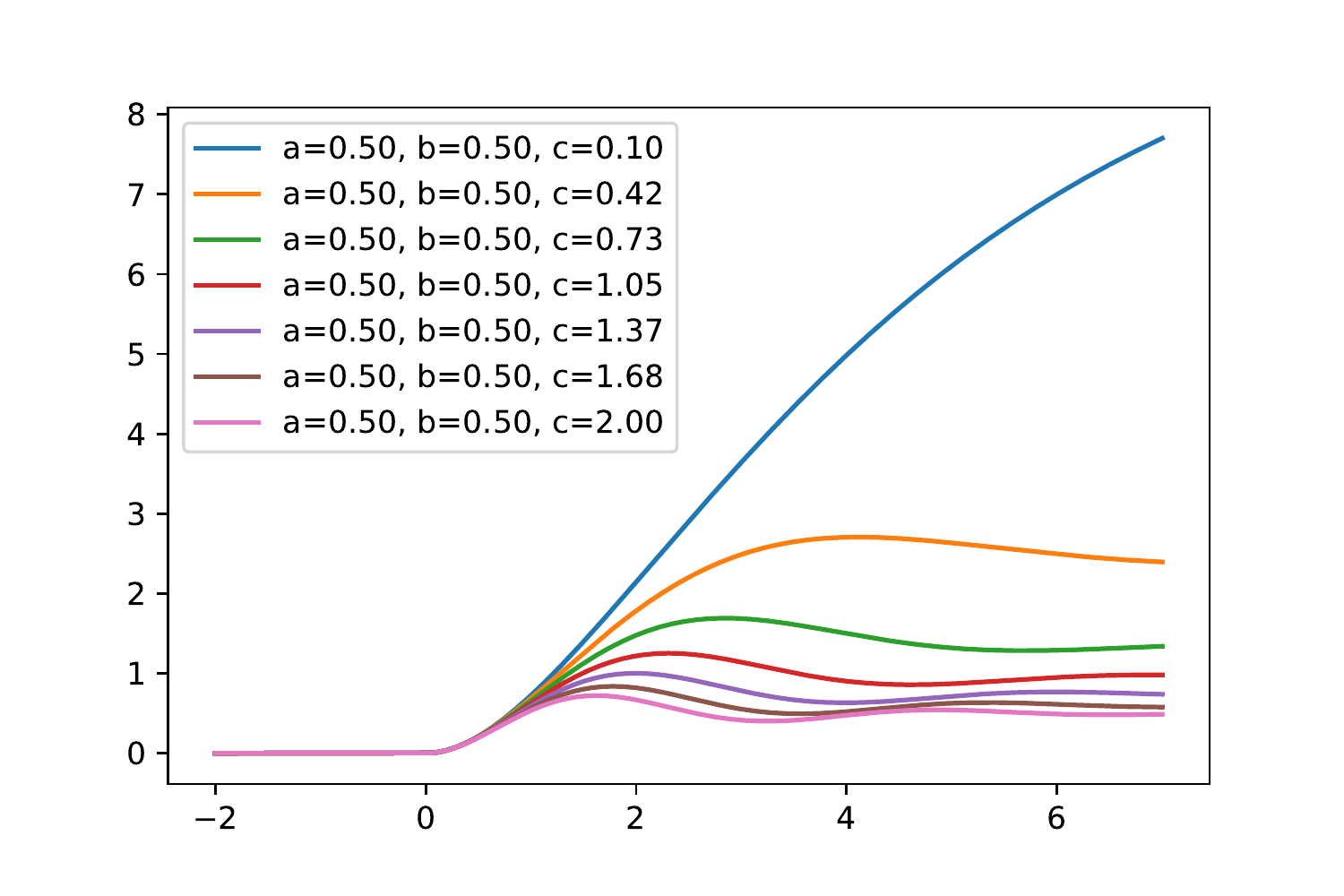}}\label{fig:var_c}
\caption{The spectra of functions generated by varying one of $a,b,c$, and
 fixing the other two with $c_1=c_2=0$.}
 \label{fig:var_coef}
\end{figure}

Figure~\ref{fig:var_coef} shows how changing a coefficient in the low
dimensional differential equation space representation will
affect the resulting functional on the manifold.

\section{Toy Datasets}\label{app:reg}
We explore the behavior of a DEU neuron on sine function toy data by training it numerous times with different initialization, each time the differential transformed itself to have a sine-like solution. This supports that DEU activation functions transform its functional form during the training process. We also demonstrated the capability of DEU to learning complex concepts, and with a significantly reduced network size.

In these tests we compared fixed activation networks using 
ReLU, LeakyReLU and SELU activations to a
significantly smaller neural network with DEUs. 

We trained on two periods of a sine function and extrapolated half of one period of unseen values (expect for DEU that we test for one full period.) As seen in Figure~\ref{sine_fcn}, a single DEU (initialized with ReLU) can learn the sine function almost perfectly, while the fixed activation baseline networks fail to fit comparable to the training data, even with a 10x larger network.


\begin{figure}[tbh]
\centering

\begin{tabular}{c c}
\includegraphics[width=0.4\linewidth]{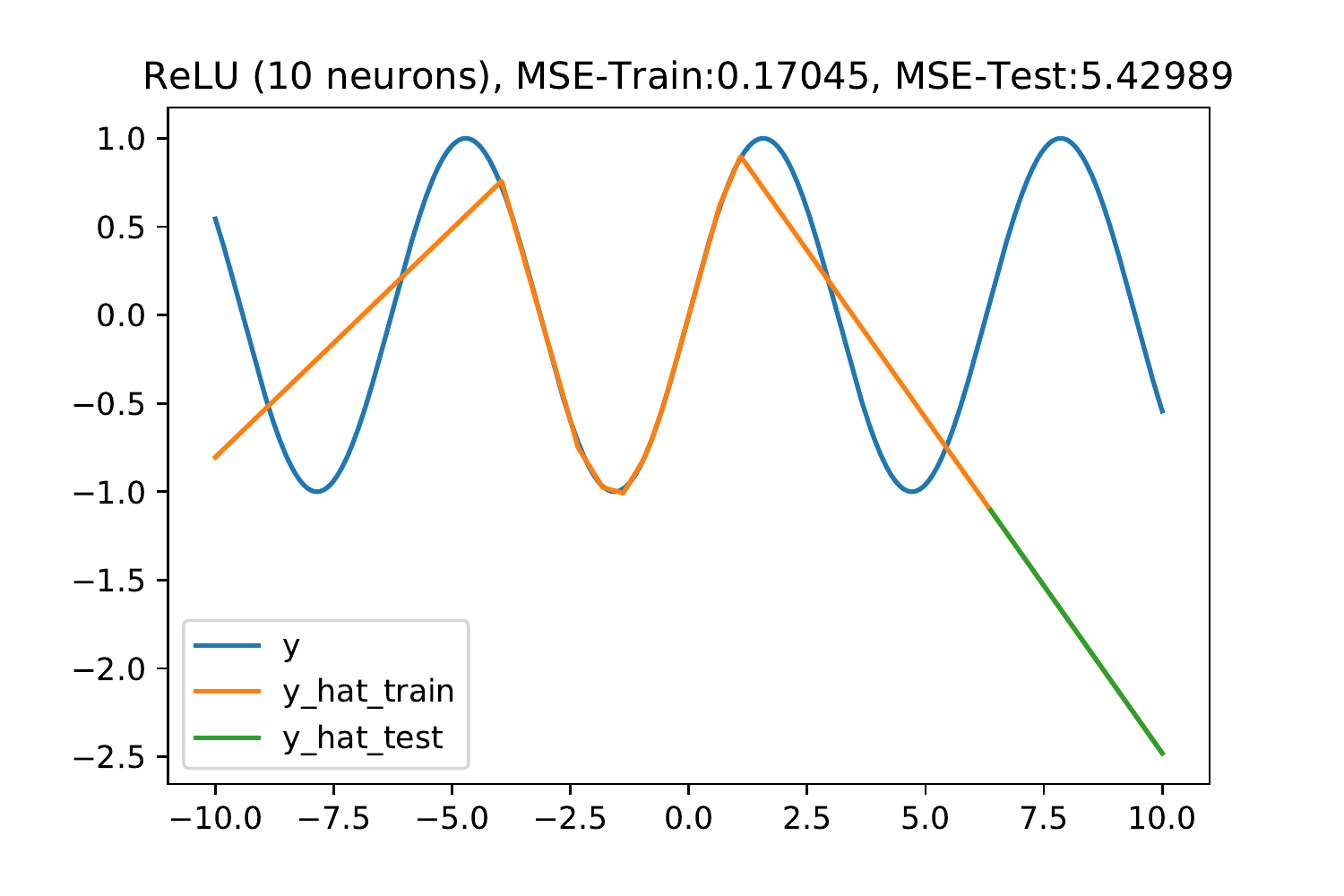}&
\includegraphics[width=0.4\linewidth]{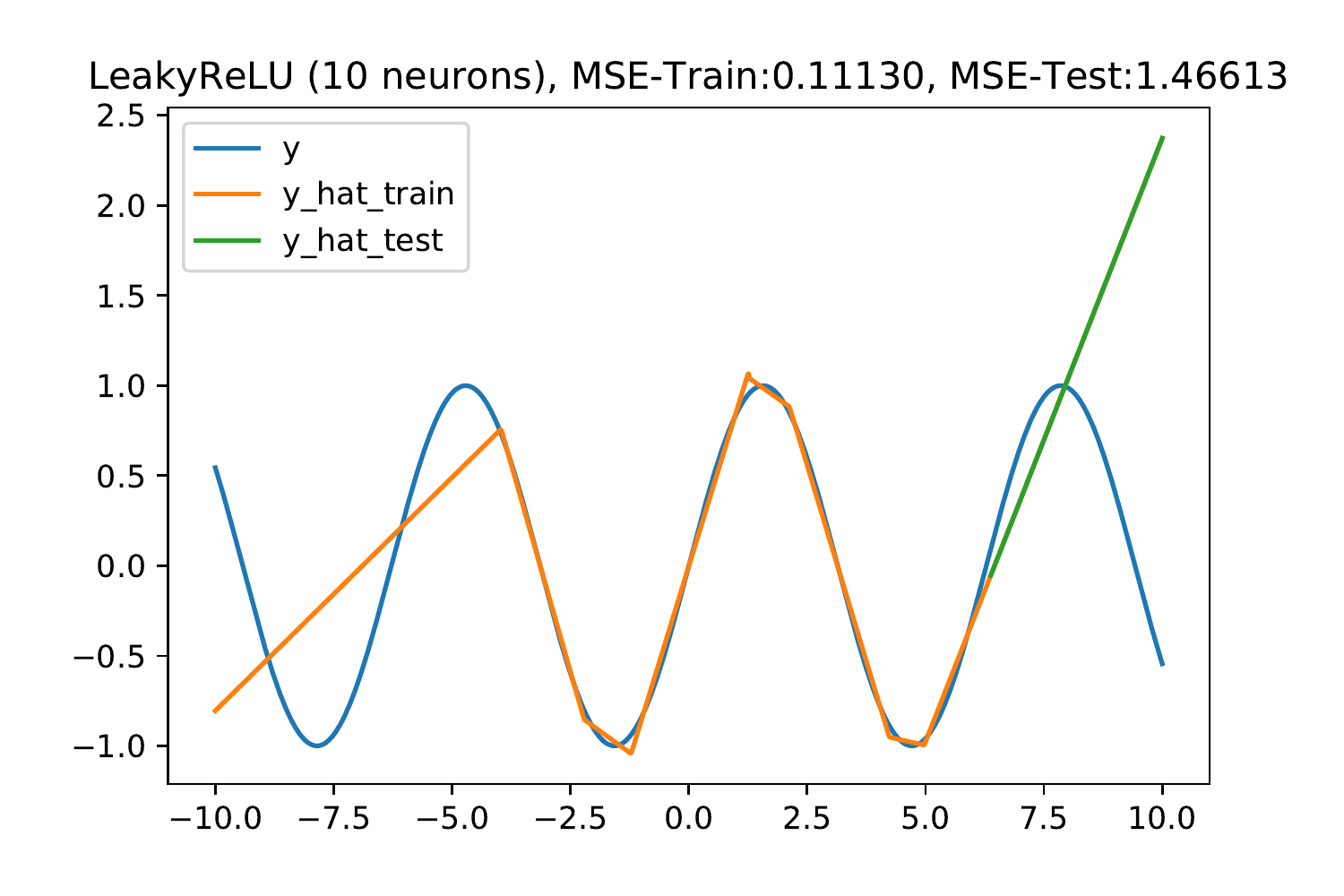}\\
\includegraphics[width=0.4\linewidth]{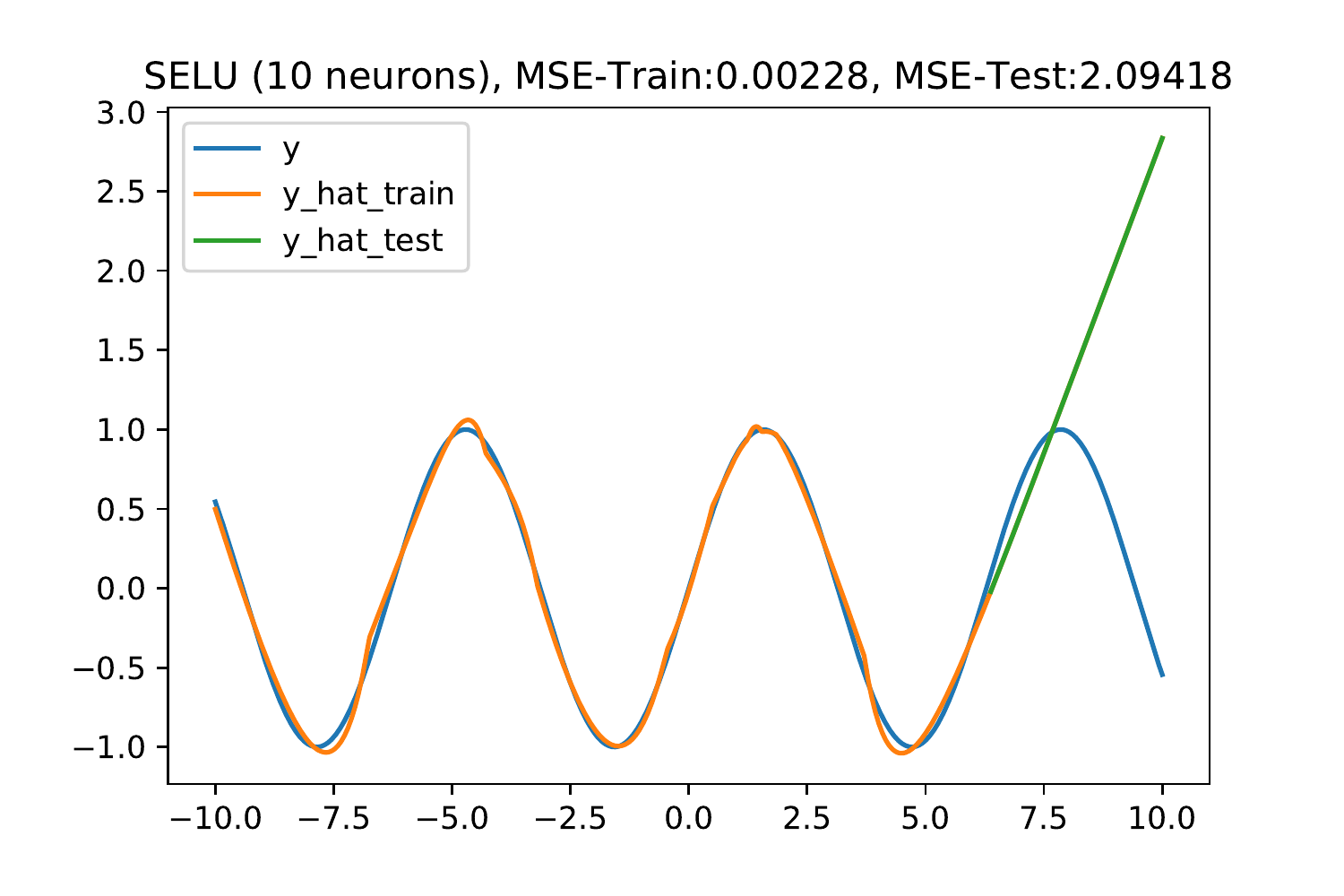}&
\includegraphics[width=0.4\linewidth]{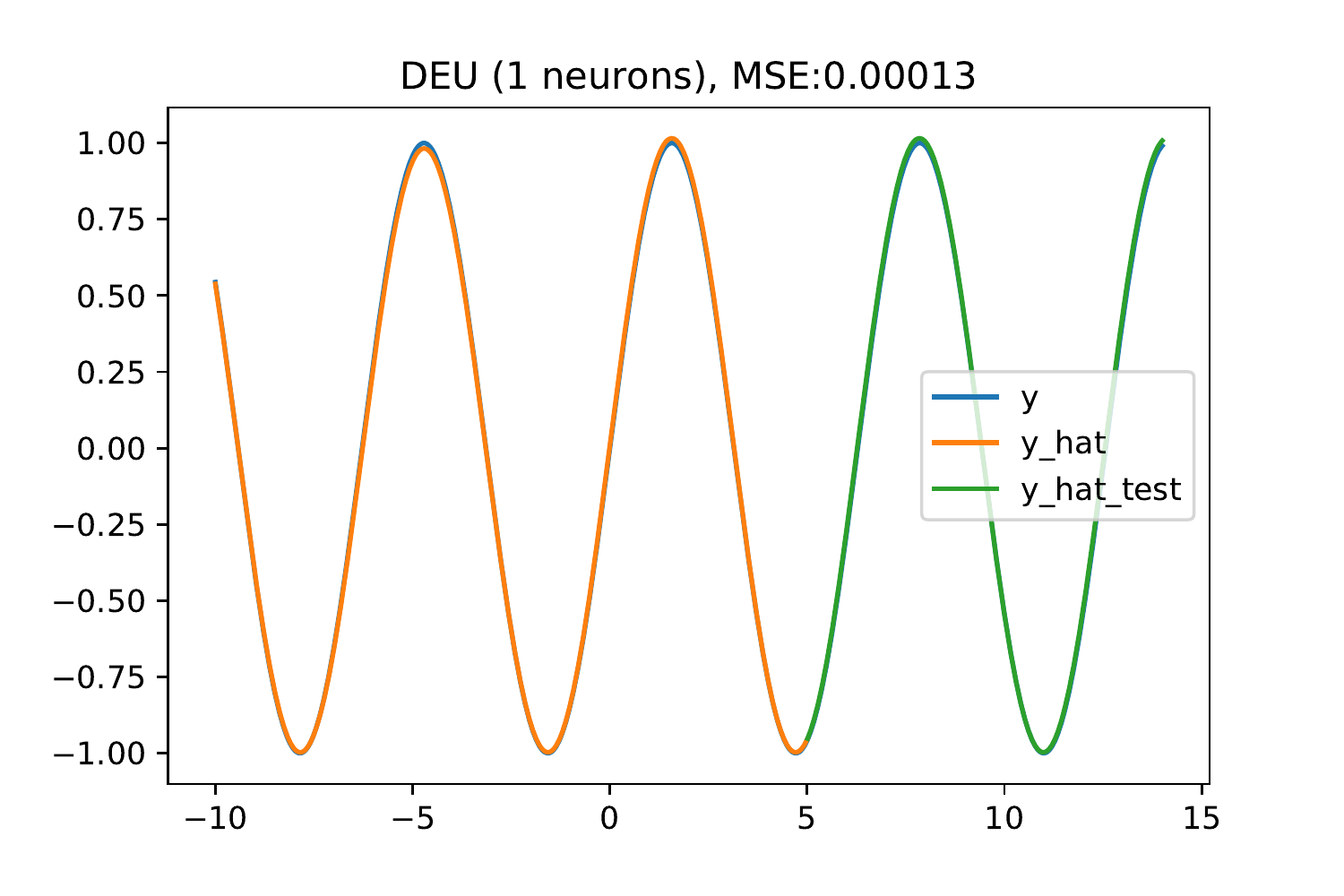}\\

\end{tabular}
\caption{Learning the sine function. The differential equation neuron learns the function significantly better than larger baselines. Also, the DEU was initialized with ReLU. It can easily transform itself to an oscillating form.
\label{sine_fcn}
}
\end{figure}


\clearpage





\end{document}